\pdfoutput=1
\RequirePackage[T1]{fontenc}
\documentclass[12pt]{article}
\usepackage{graphicx}
\graphicspath{{./}}
\usepackage[skip=1em plus 0.2em minus 0.1em]{parskip}
\usepackage{booktabs}
\usepackage[height=8.85in,width=6.45in]{geometry}

\usepackage[utf8]{inputenc}
\usepackage{amsmath}
\usepackage{amssymb}
\usepackage{mathtools}
\numberwithin{equation}{section}
\usepackage{slashed}
\usepackage{braket}
\usepackage{enumitem}
\usepackage[svgnames]{xcolor}
\usepackage[colorlinks,citecolor=DarkGreen,linkcolor=FireBrick,urlcolor=FireBrick,linktocpage,unicode]{hyperref}
\usepackage{algorithm}
\usepackage{algpseudocode}
\usepackage{xcolor}

\newcommand{\fusionstep}[1]{\textcolor{blue!75!black}{#1}}

\usepackage{tocloft}

\usepackage{footmisc}

\allowdisplaybreaks
\usepackage{graphicx}
\usepackage{tikz}
\usepackage{tikz-cd}
\usepackage{times}
 
\usepackage{bm}
\usepackage{physics}
\usepackage{xcolor}
\usepackage{natbib}
\usepackage{mdframed}
\usepackage{nicefrac}
\usepackage{booktabs}
\usepackage{lipsum}
\usepackage{titlesec}
\usepackage{wrapfig,lipsum,booktabs}
\usepackage{authblk}
\usepackage{blindtext}
\usepackage[font=small]{caption}

\usepackage{algorithm}
\usepackage{algpseudocode}
\newcommand{\commentsymbol}{//}%
\algrenewcommand\algorithmiccomment[1]{\hfill {\footnotesize \commentsymbol{} #1}}

\usepackage{titletoc}
\usepackage{todonotes}
\usepackage{authblk}
\usepackage{setspace}
\usepackage{dsfont} %

\usepackage[nameinlink,capitalize,noabbrev]{cleveref}

\usepackage{CJK}

\usepackage{array}
\usepackage{bbm}
\usepackage{makecell}

\usepackage{dashbox}
\usepackage{xcolor}
\usepackage{colortbl}
\definecolor{lightyellow}{rgb}{1.0, 0.95, 0.7}
\definecolor{Blue}{rgb}{0, 0, 0.8}
\definecolor{blue}{rgb}{0,0,1}
\definecolor{darkgreen}{rgb}{0,0.40,0}
\definecolor{firebrick}{rgb}{0.698,0.133,0.133}

\definecolor{colorA}{rgb}{1,0,0}
\definecolor{colorB}{rgb}{0,0.3,1}
\definecolor{colorC}{rgb}{0.9,0.8,0.2}
\definecolor{colorD}{rgb}{0,0.65,0}
\definecolor{lesslightgray}{rgb}{0.5,0.5,0.5}
\definecolor{light-gray}{gray}{0.95}

\let\cite\citep 

\usepackage{amsthm}
\usepackage[many]{tcolorbox}
\usepackage{etoolbox}
\BeforeBeginEnvironment{proof}{\par\vspace{-1\baselineskip}}
\AfterEndEnvironment  {proof}{\par\vspace{-1.5\baselineskip}}

\newtheoremstyle{theoremstyle}
  {.5\baselineskip} %
  {.5\baselineskip} %
  {}                  %
  {}                  %
  {\bfseries}        %
  {.}                 %
  {1em}               %
  {}                  %

\theoremstyle{theoremstyle}
\newtheorem{theorem}{Theorem}[section]

\tcolorboxenvironment{theorem}{
  breakable,
  colback=black!10,
  colframe=white,%
  width=\dimexpr\linewidth+10pt\relax,%
  enlarge left by=-5pt,%
  enlarge right by=-5pt,%
  boxsep=5pt,%
  boxrule=0pt,
  left=0pt,right=0pt,top=0pt,bottom=0pt,
  sharp corners,
  before skip=0.5\baselineskip, %
  after skip=0.5\baselineskip,  %
  fonttitle=\bfseries, %
  coltitle=black %
}
\tcolorboxenvironment{remark}{
  blanker,
  breakable,
  before skip=.8\baselineskip,  %
  after  skip=.8\baselineskip   %
}

\tcolorboxenvironment{proposition}{
  breakable,
  colback=black!10,
  colframe=white,%
  width=\dimexpr\linewidth+10pt\relax,%
  enlarge left by=-5pt,%
  enlarge right by=-5pt,%
  boxsep=5pt,%
  boxrule=0pt,
  left=0pt,right=0pt,top=0pt,bottom=0pt,
  sharp corners,
  before skip=0.5\baselineskip, %
  after skip=0.5\baselineskip,  %
  fonttitle=\bfseries, %
  coltitle=black %
}

\tcolorboxenvironment{lemma}{
  breakable,
  colback=black!10,
  colframe=white,%
  width=\dimexpr\linewidth+10pt\relax,%
  enlarge left by=-5pt,%
  enlarge right by=-5pt,%
  boxsep=5pt,%
  boxrule=0pt,
  left=0pt,right=0pt,top=0pt,bottom=0pt,
  sharp corners,
  before skip=0.5\baselineskip, %
  after skip=0.5\baselineskip,  %
  fonttitle=\bfseries, %
  coltitle=black %
}

\tcolorboxenvironment{corollary}{
  breakable,
  colback=black!10,
  colframe=white,%
  width=\dimexpr\linewidth+10pt\relax,%
  enlarge left by=-5pt,%
  enlarge right by=-5pt,%
  boxsep=5pt,%
  boxrule=0pt,
  left=0pt,right=0pt,top=0pt,bottom=0pt,
  sharp corners,
  before skip=0.5\baselineskip, %
  after skip=0.5\baselineskip,  %
  fonttitle=\bfseries, %
  coltitle=black %
}

\tcolorboxenvironment{definition}{
  breakable,
  colback=black!10,
  colframe=white,%
  width=\dimexpr\linewidth+10pt\relax,%
  enlarge left by=-5pt,%
  enlarge right by=-5pt,%
  boxsep=5pt,%
  boxrule=0pt,
  left=0pt,right=0pt,top=0pt,bottom=0pt,
  sharp corners,
  before skip=0.5\baselineskip, %
  after skip=0.5\baselineskip,  %
  fonttitle=\bfseries, %
  coltitle=black %
}
\tcolorboxenvironment{assumption}{
  breakable,
  colback=black!10,
  colframe=white,%
  width=\dimexpr\linewidth+10pt\relax,%
  enlarge left by=-5pt,%
  enlarge right by=-5pt,%
  boxsep=5pt,%
  boxrule=0pt,
  left=0pt,right=0pt,top=0pt,bottom=0pt,
  sharp corners,
  before skip=0.5\baselineskip, %
  after skip=0.5\baselineskip,  %
  fonttitle=\bfseries, %
  coltitle=black %
}

\crefname{theorem}{Theorem}{Theorems}
\crefname{proposition}{Proposition}{Propositions}
\crefname{lemma}{Lemma}{Lemmas}
\crefname{corollary}{Corollary}{Corollaries}
\crefname{definition}{Definition}{Definitions}
\crefname{assumption}{Assumption}{Assumptions}
\crefname{remark}{Remark}{Remarks}
\crefname{problem}{Problem}{Problems}
\crefname{property}{Property}{property}

\tcolorboxenvironment{hypothesis}{
  breakable,
  colback=black!10,
  colframe=white,%
  width=\dimexpr\linewidth+10pt\relax,%
  enlarge left by=-5pt,%
  enlarge right by=-5pt,%
  boxsep=5pt,%
  boxrule=0pt,
  left=0pt,right=0pt,top=0pt,bottom=0pt,
  sharp corners,
  before skip=0.5\baselineskip, %
  after skip=0.5\baselineskip,  %
  fonttitle=\bfseries, %
  coltitle=black %
}
\crefname{hypothesis}{Hypothesis}{Hypothesises}

\tcolorboxenvironment{fact}{
  breakable,
  colback=black!10,
  colframe=white,%
  width=\dimexpr\linewidth+10pt\relax,%
  enlarge left by=-5pt,%
  enlarge right by=-5pt,%
  boxsep=5pt,%
  boxrule=0pt,
  left=0pt,right=0pt,top=0pt,bottom=0pt,
  sharp corners,
  before skip=0.5\baselineskip, %
  after skip=0.5\baselineskip,  %
  fonttitle=\bfseries, %
  coltitle=black %
}
\crefname{fact}{Fact}{Facts}

\tcolorboxenvironment{example}{
  breakable,
  colback=black!10,
  colframe=white,%
  width=\dimexpr\linewidth+10pt\relax,%
  enlarge left by=-5pt,%
  enlarge right by=-5pt,%
  boxsep=5pt,%
  boxrule=0pt,
  left=0pt,right=0pt,top=0pt,bottom=0pt,
  sharp corners,
  before skip=0.5\baselineskip, %
  after skip=0.5\baselineskip,  %
  fonttitle=\bfseries, %
  coltitle=black %
}
\crefname{example}{Example}{Examples}

\tcolorboxenvironment{question}{
  breakable,
  colback=black!10,
  colframe=white,%
  width=\dimexpr\linewidth+10pt\relax,%
  enlarge left by=-5pt,%
  enlarge right by=-5pt,%
  boxsep=5pt,%
  boxrule=0pt,
  left=0pt,right=0pt,top=0pt,bottom=0pt,
  sharp corners,
  before skip=0.5\baselineskip, %
  after skip=0.5\baselineskip,  %
  fonttitle=\bfseries, %
  coltitle=black %
}

\crefname{question}{Question}{Questions}
\numberwithin{equation}{section}
\numberwithin{theorem}{section}
\numberwithin{proposition}{section}
\numberwithin{definition}{section}
\numberwithin{lemma}{section}
\numberwithin{assumption}{section}
\numberwithin{remark}{section}

\usepackage{lipsum}

\makeatletter
\let\save@mathaccent\mathaccent
\newcommand*\if@single[3]{%
    \setbox0\hbox{${\mathaccent"0362{#1}}^H$}%
    \setbox2\hbox{${\mathaccent"0362{\kern0pt#1}}^H$}%
    \ifdim\ht0=\ht2 #3\else #2\fi
}
\newcommand*\rel@kern[1]{\kern#1\dimexpr\macc@kerna}
\newcommand*\widebar[1]{\@ifnextchar^{{\wide@bar{#1}{0}}}{\wide@bar{#1}{1}}}
\newcommand*\wide@bar[2]{\if@single{#1}{\wide@bar@{#1}{#2}{1}}{\wide@bar@{#1}{#2}{2}}}
\newcommand*\wide@bar@[3]{%
    \begingroup
    \def\mathaccent##1##2{%
        \let\mathaccent\save@mathaccent
        \if#32 \let\macc@nucleus\first@char \fi
        \setbox\z@\hbox{$\macc@style{\macc@nucleus}_{}$}%
        \setbox\tw@\hbox{$\macc@style{\macc@nucleus}{}_{}$}%
        \dimen@\wd\tw@
        \advance\dimen@-\wd\z@
        \divide\dimen@ 3
        \@tempdima\wd\tw@
        \advance\@tempdima-\scriptspace
        \divide\@tempdima 10
        \advance\dimen@-\@tempdima
        \ifdim\dimen@>\z@ \dimen@0pt\fi
        \rel@kern{0.6}\kern-\dimen@
        \if#31
        \overline{\rel@kern{-0.6}\kern\dimen@\macc@nucleus\rel@kern{0.4}\kern\dimen@}%
        \advance\dimen@0.4\dimexpr\macc@kerna
        \let\final@kern#2%
        \ifdim\dimen@<\z@ \let\final@kern1\fi
        \if\final@kern1 \kern-\dimen@\fi
        \else
        \overline{\rel@kern{-0.6}\kern\dimen@#1}%
        \fi
    }%
    \macc@depth\@ne
    \let\math@bgroup\@empty \let\math@egroup\macc@set@skewchar
    \mathsurround\z@ \frozen@everymath{\mathgroup\macc@group\relax}%
    \macc@set@skewchar\relax
    \let\mathaccentV\macc@nested@a
    \if#31
    \macc@nested@a\relax111{#1}%
    \else
    \def\gobble@till@marker##1\endmarker{}%
    \futurelet\first@char\gobble@till@marker#1\endmarker
    \ifcat\noexpand\first@char A\else
    \def\first@char{}%
    \fi
    \macc@nested@a\relax111{\first@char}%
    \fi
    \endgroup
    }
\makeatother
\let\bar\widebar
\usepackage{longtable}

\setlist[itemize,enumerate]{
  parsep=\parskip,                                   %
  itemsep=\dimexpr .3em - \parskip\relax plus 2pt,   %
  topsep=\dimexpr 6pt - \parskip\relax plus 1pt minus 1pt,
  partopsep=0pt,
  listparindent=\parindent
}

\begin{document}
\begin{titlepage}

\begin{flushright}
Last Update: \today
\end{flushright}

\vskip 2.5em
\begin{center}

{
\LARGE \bfseries %
\begin{spacing}{1.15} %
Towards Generalizing Pretrained All-Atom Representations to Unified Binder Design %
\end{spacing}
}

\vskip 1.5em

Maojiang Su$^{\dagger*}$
\quad
Jiale Chen$^{\S*}$
\quad
Qijie Zhu$^{\S*}$
\quad
Shang Wu$^{\dagger}$
\quad
Jiayi Wang$^{\dagger}$
\\[0.35em]
Jihai Zhao$^{\dagger}$
\quad
Guo Ye$^{\dagger}$
\quad
Yibo Wen$^{\dagger}$
\quad
Han Liu$^{\dagger\S}$

\def\thefootnote{*}
\footnotetext{
Equal contribution.
Project page:
\href{https://jiale0402.github.io/all-atom-omni/}
{jiale0402.github.io/all-atom-omni}.
}

\vskip 1.2em

{\small
\begin{tabular}{ll}
 $^\dagger\;$Center for Foundation Models and Generative AI, Northwestern University\\
 \hphantom{$^\ddag\;$}Department of Computer Science, Northwestern University\\

 $^\S\;$Department of Statistics and Data Science, Northwestern University
\end{tabular}}

\vskip 1.2cm

\end{center}

PairFormer-based all-atom models demonstrate strong capabilities in biomolecular binder design.
However, existing methods often use all-atom representations for modality-specific binder design, such as protein and peptide design.
In contrast, recent geometric latent diffusion models pursue unified binder design by representing molecular blocks from different modalities in a shared latent space.
These complementary strengths motivate us to investigate how to generalize the use of pretrained all-atom representations to unified binder design.
We introduce \textbf{PairFuse}, a parameter-efficient method that enables this generalization through representation fusion. Specifically, we use lightweight layerwise adapters to integrate pretrained all-atom representations into an independently pretrained geometric latent diffusion model. 
Our adapters integrate pretrained all-atom representations through two mechanisms: enriching the diffusion model’s molecular block features and incorporating pairwise structural information into the geometric diffusion model's denoiser.
Theoretically, we prove that the augmented denoiser preserves geometry and pairwise attention fusion provides additional attention expressivity beyond feature-only fusion.
Experimentally, across antibody, peptide, and small molecule binder-design benchmarks, PairFuse achieves stronger performance than the pretrained diffusion model and existing representation-fusion baselines.

\end{titlepage}

\section{Introduction}
\label{sec:introduction}

Recent PairFormer-based approaches demonstrate strong capabilities in biomolecular binder design, including RFdiffusion3, BoltzDesign1, BoltzGen, and PPIFlow
\citep{butcher2025rfdiffusion3,cho2025boltzdesign1,stark2025boltzgen,yu2026ppiflow}.
However, existing methods often use all-atom representations for modality-specific binder design, such as protein and peptide design.
In parallel, geometric latent diffusion models pursue unified binder design by representing molecular blocks from different modalities in a shared latent space.
Building on this formulation, UniMoMo and AnewOmni support binder design across different molecular classes
\citep{kong2025unimomo,kong2026anewomni}.

These advances motivate us to generalize pretrained all-atom representations to unified binder design through parameter-efficient representation fusion. Existing methods, including IP-Adapter \cite{ye2023ipadapter}, LLM4GEN \cite{liu2025llm4gen}, and EnzyAdapter \cite{song2025enzycontrol}, incorporate pretrained representations into generative models through lightweight adaptation. 
However, these methods focus on feature-level conditioning and do not explicitly address how pairwise structural information can be transferred to guide geometric interactions during denoising. The challenge is therefore to integrate pretrained information about both molecular components and their interactions into an independently pretrained geometric latent diffusion model.

To address these challenges, we introduce \textbf{PairFuse}, a parameter-efficient method for extending pretrained all-atom representations to unified binder design through representation fusion. We instantiate PairFuse on AnewOmni~\citep{kong2026anewomni}, an independently pretrained geometric latent diffusion model, using single and pair representations extracted from a pretrained BoltzGen PairFormer~\citep{stark2025boltzgen}. Our lightweight layerwise adapters fuse these representations into the geometric diffusion denoiser through two mechanisms. First, cross-attention enriches the diffusion model's invariant molecular block features with pretrained structural information. Second, additive attention biases incorporate pretrained pairwise relationships into the Transformer's attention to guide denoising. This design preserves the original equivariant vector operations while keeping both pretrained models frozen and training only the adapters.

A key consideration in representation fusion for geometric latent diffusion is preserving rotational equivariance when incorporating pretrained structural information into native geometric attention.
The geometric denoiser processes three-dimensional molecular structures, and its coordinate predictions transform consistently under rotations
\citep{fuchs2020se3transformer,satorras2021egnn,jing2021gvp,kong2024get}.
PairFuse incorporates pretrained all-atom representations into rotation-invariant block features and native geometric attention without modifying the equivariant vector operations.
We prove that the augmented denoiser preserves rotational equivariance under rotation-invariant transferred representations.
This guarantee holds for arbitrary learned adapter parameters, establishing the geometric compatibility of PairFuse.

We evaluate PairFuse on antibody, peptide, and small-molecule binder-design benchmarks:
RAbD~\citep{adolfbryfogle2018rabd},
LNR~\citep{tsaban2022peptide},
and CrossDocked2020~\citep{francoeur2020crossdocked}.
PairFuse improves structural accuracy and interface quality on antibody and peptide design and achieves competitive performance on small-molecule design.
These results support the effectiveness of parameter-efficient representation fusion for unified binder design.

\paragraph{Contributions.}
Our contributions are three-fold.
\begin{itemize}
    \item \textbf{Pretrained All-Atom Representation Fusion.}
    We propose PairFuse, a parameter-efficient method that extends the use of pretrained all-atom representations to unified binder design.
    Our layerwise adapters integrate pretrained single and pair representations into the evolving invariant block features and native geometric attention of a geometric latent diffusion model.
    We freeze both pretrained models and retain the original latent space, sampling procedure, and all-atom decoder
    (\cref{sec:methodology}).

    \item \textbf{Theoretical Analysis.}
    We prove that PairFuse preserves rotational equivariance under rotation-invariant transferred representations
    (\cref{thm:geometry-compatible-transfer}).
    We further characterize the expressivity of pairwise attention fusion.
    Under a stated rank condition, direct attention fusion expresses attention patterns that scalar-feature-only fusion cannot reproduce through fixed-width native query--key projections
    (\cref{thm:pairwise-expressivity}).

    \item \textbf{Empirical Evaluation.}
    We evaluate PairFuse across antibody, peptide, and small-molecule binder design using shared fusion checkpoints.
    PairFuse improves structural accuracy and interface quality on antibody and peptide benchmarks and achieves favorable results on small-molecule design
    (\cref{sec:lnr-results}).
    We further investigate the contributions of single and pair representations and pairwise attention fusion through ablation studies
    (\cref{app:ablations}).
    We report experimental details, statistical analyses, and computational efficiency
    (\cref{app:experimental-details}).
\end{itemize}

\section{Preliminaries}
\label{sec:preliminaries}

We introduce the unified biomolecular design setting, geometric latent diffusion, the Equivariant Pretrained Transformer (EPT), and pretrained all-atom representations.

\paragraph{Unified Binder Design and Latent Representation.}
Unified geometric latent diffusion models, such as UniMoMo and AnewOmni~\citep{kong2025unimomo,kong2026anewomni}, represent different molecular classes in a shared latent space.
Let $c$ denote the design condition, including the target, design region, binder type, and prescribed constraints.
We seek to generate a binder $x$ from the conditional distribution $p_\theta(x\mid c)$.
Specifically, each molecular block (e.g., an amino acid residue in a peptide or a molecular fragment in a small-molecule ligand) is represented by an invariant feature and a three-dimensional coordinate, yielding
\begin{align*}
    z=(h,q),
    \qquad
    h\in\mathbb R^{n\times d},
    \quad
    q\in\mathbb R^{n\times3},
\end{align*}
for $n$ molecular blocks.
A pretrained all-atom decoder $D$ reconstructs the binder from the generated latent blocks and design condition $c$.

\paragraph{Geometric Latent Diffusion.}
Let $z_t=(h_t,q_t)$ denote the joint latent state of generated and context blocks at diffusion step $t$, with $z_0$ denoting the clean state.
For each generated block $i$, the forward noising process is
\begin{align*}
    h_{t,i}
    &=
    \alpha_t h_{0,i}+\sigma_t\epsilon_{h,i},
    \\
    q_{t,i}
    &=
    \alpha_t q_{0,i}+\sigma_t\epsilon_{q,i},
\end{align*}
where $\alpha_t^2+\sigma_t^2=1$ and the noise terms are standard Gaussian vectors of matching dimensions.
Context blocks remain fixed at their clean states.
Let $a$ denote the invariant conditioning embeddings constructed from the design information.
The denoiser predicts noise on generated blocks:
\begin{align*}
    \epsilon_\theta(z_t,t\mid a)
    =
    (\epsilon_{\theta,h},\epsilon_{\theta,q}).
\end{align*}
Sampling applies the reverse-diffusion updates with context blocks fixed.
A geometric denoiser is required to preserve rotational equivariance.
For any global rotation $R\in\mathrm{SO}(3)$, define
$R\cdot(h,q)=(h,qR^\top)$, where coordinates are stored row-wise.
The denoiser should satisfy
\begin{equation}
\begin{aligned}
    \epsilon_{\theta,h}(R\cdot z_t,t\mid a)
    &=
    \epsilon_{\theta,h}(z_t,t\mid a),
    \\
    \epsilon_{\theta,q}(R\cdot z_t,t\mid a)
    &=
    R\cdot\epsilon_{\theta,q}(z_t,t\mid a).
\end{aligned}
\label{eq:geometric-symmetry}
\end{equation}

\paragraph{EPT Denoiser.}
We consider an EPT-based geometric denoiser~\citep{jiao2026equivariant}.
The native input transformation and geometric initialization are
\begin{align*}
    U^{\mathrm{in}}
    &=
    \operatorname{InputEmbedding}_{\theta}(h_t,a,t),
    \\
    (U^{(0)},\vec V^{(0)})
    &=
    \operatorname{GraphInitialization}_{\theta}
    (U^{\mathrm{in}},q_t),
\end{align*}
where $U^{\mathrm{in}}\in\mathbb R^{n\times d_{\mathrm E}}$ and $d_{\mathrm E}$ is the hidden width.
At layer $\ell$, EPT maintains scalar and vector features
\begin{align*}
    U^{(\ell)}\in\mathbb R^{n\times d_{\mathrm E}},
    \qquad
    \vec V^{(\ell)}\in\mathbb R^{n\times3\times d_{\mathrm E}}.
\end{align*}
Each EPT layer performs equivariant self-attention and geometric feed-forward updates.
For attention head $k$, the native logits
$L_k^{(\ell)}\in\mathbb R^{n\times n}$
combine scalar query--key similarities with geometric and edge-feature biases.
The resulting attention weights aggregate scalar and vector features, while the native output heads predict noise on generated blocks.
We summarize the native EPT computation in \cref{alg:native-ept}.
Under global rotations, the scalar features and attention logits of EPT remain invariant, and the vector features transform equivariantly:
\begin{align*}
    U^{(\ell)}
    &\mapsto U^{(\ell)},
    \qquad
    L_k^{(\ell)}
    \mapsto L_k^{(\ell)},
    \\
    \vec V^{(\ell)}
    &\mapsto R\cdot\vec V^{(\ell)}.
\end{align*}
These transformation properties ensure that the EPT denoiser satisfies \eqref{eq:geometric-symmetry}.

\paragraph{Pretrained All-Atom Representations.}
PairFormer-based all-atom binder-design models, such as RFdiffusion3 and BoltzGen~\citep{butcher2025rfdiffusion3,stark2025boltzgen}, learn single and pair representations that capture token-level structural information and pairwise interactions.
Given the visible design condition $c$, we denote these pretrained representations by $r(c)=(S,P)$, where $S$ contains token-level single representations and $P$ contains pair representations encoding learned relations between token pairs.

\textbf{
\begin{figure}[t]
    \centering
    \includegraphics[width=\linewidth]{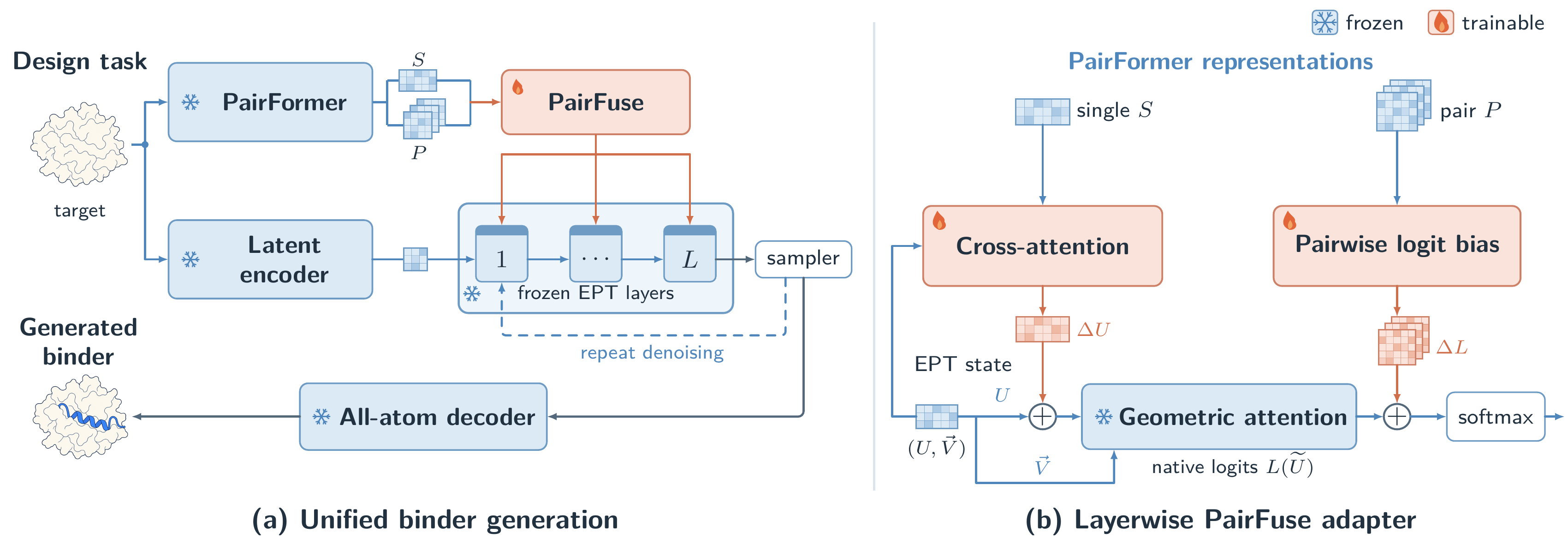}
    \caption{
    Overview of PairFuse.
    A frozen BoltzGen PairFormer extracts pretrained single and pair representations, while AnewOmni provides the geometric latent diffusion model.
    At each denoising step, trainable layerwise adapters integrate these representations into the invariant scalar features and native geometric attention logits of EPT.
    The original sampler and all-atom decoder generate the final binder.
    }
    \label{fig:method-overview}
\end{figure}
}

\section{Pretrained All-Atom Representation Fusion}
\label{sec:methodology}

We introduce \textbf{PairFuse}, a parameter-efficient method that integrates pretrained all-atom representations into geometric latent diffusion through layerwise fusion.
We instantiate PairFuse on AnewOmni~\citep{kong2026anewomni}, using single and pair representations extracted from a pretrained BoltzGen PairFormer~\citep{stark2025boltzgen}.
As illustrated in \cref{fig:method-overview}, our layerwise adapters integrate these representations into the evolving invariant scalar features and native geometric attention logits of EPT.
We freeze both pretrained models and train only the adapters.
We first introduce the layerwise representation fusion mechanism
(\cref{sec:representation-transfer}),
then establish its geometric compatibility
(\cref{sec:geometry-compatibility}),
and finally present the training and inference procedures
(\cref{sec:training-guidance}).

\subsection{Layerwise All-Atom Representation Fusion}
\label{sec:representation-transfer}

PairFuse integrates pretrained PairFormer representations into two components of each EPT layer: invariant scalar features and native attention logits.
We introduce scalar-feature fusion and pairwise attention fusion to incorporate pretrained structural information into these components.

\paragraph{Scalar-Feature Fusion.}
Given the design condition $c$, we extract pretrained single and pair representations $r(c)=(S,P)$ from a frozen PairFormer.
We mask the unknown binder sequence and coordinates during representation extraction.
We align the PairFormer tokens with the molecular blocks of the geometric denoiser through molecular correspondence.
Let $m_i\in\{0,1\}$ indicate whether block $i$ has a corresponding PairFormer token.
We use block-aligned representations $S$ and $P$, and apply correspondence masks wherever PairFormer representations are unavailable.

At each EPT layer, PairFuse uses the current scalar features $U$ as queries to attend to pretrained single representations.
A trainable encoder $f_s$ produces $\bar S=f_s(S)$.
For each attention head, we compute
\begin{align*}
    Q=UW_Q,
    \qquad
    K=\bar S W_K,
    \qquad
    V=\bar S W_V,
\end{align*}
where $W_Q$, $W_K$, and $W_V$ are trainable query, key, and value projection matrices, respectively.
We then compute the cross-attention weights with an additional pair-dependent bias:
\begin{align}
    A
    =
    \operatorname{MaskedSoftmax}_m
    \left(
        \frac{QK^\top}{\sqrt{d_h}}
        +
        b(P)
    \right),
    \label{eq:feature-attention}
\end{align}
where $d_h$ denotes the attention-head dimension and $b$ is a trainable projection that maps pair representations to attention biases.
The masked softmax attends only to keys $j$ with $m_j=1$, and the pair-dependent bias contributes only when both corresponding tokens are available.

We then update the scalar features $U$ through
\begin{align}
    \Delta U
    &=
    \operatorname{Concat}_{\mathrm{heads}}(AV)W_O
    +
    \operatorname{diag}(m)\bar S W_D,
    \notag\\
    \widetilde U
    &=
    U+\Delta U,
    \label{eq:feature-fusion}
\end{align}
where $W_O$ and $W_D$ are trainable output and direct-projection matrices.
The first term transfers pretrained information through cross-attention, including to blocks without direct PairFormer correspondence.
The second term directly incorporates single representations at blocks with available PairFormer tokens.
We share $f_s$ across EPT layers and use layer-specific attention projections.

\paragraph{Pairwise Attention Fusion.}
PairFuse further integrates pretrained representations into the native geometric attention of EPT.
As described in \cref{alg:native-ept}, each EPT layer computes attention logits $L(U)$ from scalar features and geometric information.
We introduce an additive bias $\Delta L$ to modulate these logits using pretrained single and pair representations.
For query block $i$ and key block $j$, we compute
\begin{align}
    \Delta L_{ij}
    &=
    \frac{m_j}{\sqrt{d_b}}
    (U_iW^L_{\mathrm Q})
    (\operatorname{LN}(S_j)W^L_{\mathrm K})^\top+
    m_i m_j f_p(P_{ij})W^L_{\mathrm P},
    \notag\\
    \widetilde L
    &=
    L(\widetilde U)+\Delta L,
    \label{eq:attention-fusion}
\end{align}
where $d_b$ denotes the projection dimension, $\operatorname{LN}$ denotes layer normalization, and $f_p$ is a trainable pair encoder shared across EPT layers.
The matrices $W^L_{\mathrm Q}$, $W^L_{\mathrm K}$, and $W^L_{\mathrm P}$ are trainable projections that produce the query, key, and pair-dependent attention contributions, respectively.
The first term uses pretrained single representations whenever the key block has a corresponding PairFormer token.
The second term uses pair representations when both blocks have corresponding tokens.
We compute $\Delta L$ from the original scalar features $U$ and compute the native attention logits using the updated features $\widetilde U$.

The augmented attention uses $\widetilde U$ and $\widetilde L$ to aggregate scalar and vector features through the original EPT operations.
Together, the two fusion mechanisms update block representations and modulate geometric interactions without modifying the native equivariant vector operations.
PairFuse reuses the pretrained PairFormer representations across diffusion steps and updates its fusion outputs using the evolving EPT features.
We summarize the augmented EPT architecture and fusion modules in
\cref{alg:layerwise-fusion,alg:fusion-modules}.
We denote the resulting noise prediction on generated blocks by
\begin{align}
    \epsilon_{\mathrm{aug}}(z_t,t\mid a,S,P)
    =
    \epsilon_{\theta,\phi}(z_t,t\mid a,S,P),
    \label{eq:augmented-denoiser}
\end{align}
where $\theta$ denotes the frozen denoiser parameters and $\phi$ collects all trainable fusion parameters.

\subsection{Theory}
\label{sec:geometry-compatibility}

We analyze PairFuse from two perspectives.
First, we prove that scalar-feature fusion and pairwise attention fusion preserve the denoiser's rotational equivariance (\cref{thm:geometry-compatible-transfer}).
Second, we study how pairwise attention fusion enhances attention expressivity (\cref{thm:pairwise-expressivity}).

PairFuse updates the denoiser's invariant scalar features and native geometric attention logits.
We first prove that the augmented denoiser preserves the geometry of latent diffusion \eqref{eq:geometric-symmetry}.

\begin{theorem}[Geometric Compatibility]
\label{thm:geometry-compatible-transfer}
Suppose the native conditioning features $a$, the transferred representations $S,P$, and the correspondence mask $m$ are rotation invariant.
Assume that the native input transformations and output heads preserve scalar invariance and vector equivariance, and that each EPT layer preserves these properties for invariant attention weights.
Then, for any fusion parameters $\phi$ and every $R\in\mathrm{SO}(3)$, the augmented denoiser satisfies
\begin{align}
    \epsilon_{\mathrm{aug},h}(R\cdot z_t,t\mid a,S,P)
    &=
    \epsilon_{\mathrm{aug},h}(z_t,t\mid a,S,P),
    \label{eq:augmented-invariance}
    \\
    \epsilon_{\mathrm{aug},q}(R\cdot z_t,t\mid a,S,P)
    &=
    R\cdot\epsilon_{\mathrm{aug},q}(z_t,t\mid a,S,P).
    \label{eq:augmented-equivariance}
\end{align}
Here $R\cdot z_t$ rotates generated and context coordinates jointly, together with any native geometric prompts omitted from the notation.
\end{theorem}

\begin{proof}[Proof Sketch]
At each layer, invariant $U,S,P,m$ yield invariant updates $\Delta U$ and $\Delta L$ in
\eqref{eq:feature-fusion} and \eqref{eq:attention-fusion}.
The updated scalar features and attention weights therefore remain invariant.
The native geometric operations preserve scalar invariance and vector equivariance under these updates.
Induction over the layers  establishes the result.
We provide the complete proof in \cref{app:geometry-proof}.
\end{proof}

\Cref{thm:geometry-compatible-transfer} shows that PairFuse preserves rotational equivariance for arbitrary adapter parameters.
The fusion adapters therefore learn to incorporate pretrained structural information without relearning the denoiser's rotational symmetry.

Scalar-feature fusion affects native attention through its query and key projections.
We examine whether this indirect mechanism can replace the direct attention-logit updates in \eqref{eq:attention-fusion}.
Our analysis builds on the low-rank structure of dot-product attention~\citep{bhojanapalli2020lowrank}.
Fix a diffusion state, an EPT layer, and one native attention head.
Consider $n$ valid blocks with all query-key pairs allowed.
We write the native logits as
\begin{equation}
    L(U)
    =
    Q_{\mathrm{nat}}(U)
    K_{\mathrm{nat}}(U)^\top
    +G,
    \label{eq:native-logit-factorization}
\end{equation}
where
$Q_{\mathrm{nat}}(U),K_{\mathrm{nat}}(U)
\in\mathbb R^{n\times d_{\mathrm{att}}}$
denote the native query and key features, and $d_{\mathrm{att}}$ is their dimension.
These functions include the native feature normalization and attention scaling.
The matrix $G\in\mathbb R^{n\times n}$ collects finite geometric biases that remain fixed under scalar-feature updates.

We allow a scalar-feature-only adapter to access the same pretrained representations $S,P$ and produce arbitrary updated scalar features $\widehat U$.
This adapter retains the native attention computation and adds no direct attention-logit bias.
Let
$\Pi_F=\operatorname{Softmax}(L(\widehat U))$
denote its attention weights, and let
$\Pi_{\mathrm{full}}=\operatorname{Softmax}(\widetilde L)$
denote the attention weights of PairFuse, with $\widetilde L$ defined in \eqref{eq:attention-fusion}.
\begin{theorem}[Expressivity of Pairwise Attention Fusion]
\label{thm:pairwise-expressivity}
Under the setting above, define
\begin{equation}
    H=I_n-\frac{1}{n}\mathbf 1\mathbf 1^\top,
    \qquad
    T=(\widetilde L-G)H,
    \label{eq:centered-attention-residual}
\end{equation}
where $I_n$ is the identity matrix and $\mathbf 1\in\mathbb R^n$ is the all-ones vector.
Every scalar-feature-only adapter satisfies
\begin{equation}
    \left\|
        \bigl(
            \log\Pi_{\mathrm{full}}-\log\Pi_F
        \bigr)H
    \right\|_F^2
    \geq
    \sum_{r>d_{\mathrm{att}}}\sigma_r(T)^2,
    \label{eq:pairwise-expressivity-bound}
\end{equation}
where $\log$ acts element-wise, $\|\cdot\|_F$ denotes the Frobenius norm, and $\sigma_r(T)$ is the $r$th largest singular value of $T$.
In particular, if
$\operatorname{rank}(T)>d_{\mathrm{att}}$,
no scalar-feature-only adapter reproduces
$\Pi_{\mathrm{full}}$ at this head and state.
\end{theorem}

\begin{proof}[Proof Sketch]
For any finite logit matrix $L$,
row-wise softmax satisfies
\[
    \bigl(\log\operatorname{Softmax}(L)\bigr)H=LH,
\]
because multiplication by $H$ removes row-wise constants.
Consequently,
\[
    \bigl(\log\Pi_{\mathrm{full}}-\log\Pi_F\bigr)H
    = T-Q_{\mathrm{nat}}(\widehat U)K_{\mathrm{nat}}(\widehat U)^\top H.
\]
The second term has rank at most $d_{\mathrm{att}}$.
The Eckart-Young theorem therefore gives
\eqref{eq:pairwise-expressivity-bound}.
A positive singular-value tail rules out identical attention weights.
We provide proof details in \cref{app:pairwise-expressivity}.
\end{proof}

\Cref{thm:pairwise-expressivity} shows that, under the stated rank condition, direct pairwise attention fusion expresses attention patterns that scalar-feature-only adapters cannot reproduce through fixed-width native query-key projections, even with access to the same $S,P$.
Scalar-feature fusion enriches block features, and pairwise attention fusion directly modulates the weights that aggregate scalar and vector values.
Together with \cref{thm:geometry-compatible-transfer}, these results motivate the joint design of PairFuse, combining additional attention expressivity with rotational equivariance.

\subsection{Training and Inference}
\label{sec:training-guidance}

We freeze the BoltzGen trunk and the pretrained AnewOmni
parameters $\theta$ and optimize only the fusion parameters
$\phi$.
Gradients propagate through the frozen denoiser to the adapters.
Let $\epsilon_h$ and $\epsilon_q$ denote the noise added to the
generated block features and coordinates, respectively.
We retain the native noise-prediction objective:
\begin{align}
    \mathcal L_{\mathrm{fusion}}(\phi)
    =
    \mathbb E
    \left[
        \frac{\lambda}{n_G}
        \left\|
            \epsilon_{\mathrm{aug},h}-\epsilon_h
        \right\|_F^2
        +
        \frac{1}{n_G}
        \left\|
            \epsilon_{\mathrm{aug},q}-\epsilon_q
        \right\|_F^2
    \right],
    \label{eq:fusion-training}
\end{align}
where $n_G$ is the number of generated blocks and $\lambda$
weights the latent-feature loss.
Predictions and noise targets are restricted to generated blocks.
Context rows of $z_t$ remain fixed, although their internal EPT
features participate in fusion.
No additional structural supervision is introduced.

At inference time, the learned fusion adapters are applied at
each EPT layer using the same operations as during training.
The PairFormer representations are cached and reused throughout
reverse diffusion.
The native sampling procedure and pretrained all-atom decoder
remain unchanged.

\section{Experimental Studies}
\label{sec:exp}
We evaluate PairFuse on antibody, peptide, and small-molecule binder design. Our goal is to assess whether pretrained all-atom representations can improve an independently pretrained geometric latent diffusion model through parameter-efficient representation fusion. We compare PairFuse against the original AnewOmni checkpoint and three adapted representation-fusion baselines. Additional experimental details are provided in \cref{app:experimental-details}.

\subsection{Experimental Setup}
\label{sec:experimental-setup}

\paragraph{Task and Benchmark.}
We evaluate HCDR3 design on RAbD~\citep{adolfbryfogle2018rabd}, peptide design on LNR~\citep{tsaban2022peptide}, and small-molecule design on CrossDocked2020~\citep{francoeur2020crossdocked}. The fixed evaluation sets contain 60 RAbD targets, 93 LNR targets, and a 100-target subset of CrossDocked2020, respectively.

\paragraph{Baselines.}
We compare against AnewOmni~\citep{kong2026anewomni}, LLM4GEN~\citep{liu2025llm4gen}, EnzyAdapter~\citep{song2025enzycontrol}, and IP-Adapter~\citep{ye2023ipadapter}. AnewOmni uses its original pretrained checkpoint without additional training. LLM4GEN and EnzyAdapter adapt the respective representation-fusion mechanisms to the AnewOmni backbone. LLM4GEN applies a single cross-attention module to update native conditioning before the unchanged denoiser. EnzyAdapter uses layerwise cross-attention, concatenation, and projection modules. IP-Adapter applies decoupled cross-attention at each EPT layer, using native queries to attend separately to projected PairFormer single representations. The donor scalar output is added before the shared frozen output projection. All fusion baselines are trained on the common biomolecular task. All trained representation-fusion methods are matched in trainable parameter count (approximately 4.03 million, with a maximum relative difference below 0.1\%), with details reported in \cref{tab:trainable-parameters}.

\paragraph{Evaluation.}
Following AnewOmni~\citep{kong2026anewomni}, we evaluate structural accuracy and amino-acid recovery for antibody and peptide design. We additionally assess interface quality through local-interface adaptations of DockQ and native-contact recall~\citep{basu2016dockq}, together with a backbone interface score based on lDDT~\citep{mariani2013lddt}. For small-molecule design, we report Vina score and Vina min~\citep{trott2010vina}, High Affinity~\citep{guan2023targetdiff}, QED~\citep{bickerton2012qed}, and Atom-type MAE~\citep{lin2024cbgbench}. Additional metrics, including steric clashes, interface RMSD, and Lipinski rule compliance~\citep{lipinski1997rules}, are reported in the appendix.

We use 30 fixed generation attempts per target. Native metrics pool finite candidate scores within each run, whereas interface and supplementary structural metrics first average finite candidates within each target and then weight targets equally. We select one checkpoint per trained run by minimizing the combined native validation loss on a 614-record validation set, weighting each modality by its number of generated blocks. The selected checkpoint is shared across all three modalities. Metric definitions and evaluation procedures are provided in \cref{app:inference-protocol}. PairFormer-representation construction and training details are provided in \cref{app:donor-construction,app:training-recipe}.

\subsection{Main Results}
\label{sec:trimodal-update}
\label{sec:lnr-results}

Tables~\ref{tab:main-rabd}, \ref{tab:main-lnr}, and
\ref{tab:main-crossdocked2020} compare PairFuse with the
pretrained AnewOmni model and three representation-fusion
baselines across antibody, peptide, and small-molecule
binder design. For each trained method, the reported results are the arithmetic mean of three seed-specific metric estimates, with equal weight assigned to each seed. AnewOmni uses its original pretrained checkpoint. The complete appendix comparison includes the same five methods.

\begin{table*}[!t]
\centering

\caption{
Antibody design results on RAbD (mean $\pm$ std over 3 random seeds).
The best and second-best mean metric values are shown in \textbf{bold} and \underline{underline}, respectively.
}
\label{tab:main-rabd}
\begingroup
\fontsize{9}{10.5}\selectfont
\setlength{\tabcolsep}{3pt}
\renewcommand{\arraystretch}{1.18}
\begin{tabular*}{\textwidth}{@{\extracolsep{\fill}}lccccc@{}}
\toprule
& \multicolumn{3}{c}{Structure and Sequence}
& \multicolumn{2}{c}{Interface Quality} \\
\cmidrule(lr){2-4}
\cmidrule(l){5-6}
Method
& Complex RMSD $\downarrow$
& RMSD $\leq 2$~\AA{} (\%) $\uparrow$
& AAR $\uparrow$
& DockQ $\uparrow$
& Contact $\uparrow$ \\
\midrule
{\fontsize{9}{10.5}\selectfont AnewOmni~\citep{kong2026anewomni}}
& {\fontsize{9}{10.5}\selectfont 2.5563}
& {\fontsize{9}{10.5}\selectfont 34.00}
& {\fontsize{9}{10.5}\selectfont 15.6755}
& {\fontsize{9}{10.5}\selectfont 0.6853}
& {\fontsize{9}{10.5}\selectfont 50.9592} \\
\midrule
{\fontsize{9}{10.5}\selectfont LLM4GEN~\citep{liu2025llm4gen}}
& {\fontsize{9}{10.5}\selectfont 2.5844$\pm$0.09}
& {\fontsize{9}{10.5}\selectfont 33.75$\pm$1.23}
& {\fontsize{9}{10.5}\selectfont \textbf{15.7987$\pm$0.24}}
& {\fontsize{9}{10.5}\selectfont 0.6837$\pm$0.00}
& {\fontsize{9}{10.5}\selectfont 50.9243$\pm$0.39} \\
{\fontsize{9}{10.5}\selectfont EnzyAdapter~\citep{song2025enzycontrol}}
& {\fontsize{9}{10.5}\selectfont \underline{2.5055$\pm$0.07}}
& {\fontsize{9}{10.5}\selectfont \underline{36.80$\pm$1.77}}
& {\fontsize{9}{10.5}\selectfont 15.6484$\pm$0.07}
& {\fontsize{9}{10.5}\selectfont \underline{0.6888$\pm$0.01}}
& {\fontsize{9}{10.5}\selectfont \underline{51.0443$\pm$0.46}} \\
{\fontsize{9}{10.5}\selectfont IP-Adapter~\citep{ye2023ipadapter}}
& {\fontsize{9}{10.5}\selectfont 2.5416$\pm$0.01}
& {\fontsize{9}{10.5}\selectfont 35.59$\pm$0.66}
& {\fontsize{9}{10.5}\selectfont 15.4042$\pm$0.17}
& {\fontsize{9}{10.5}\selectfont 0.6854$\pm$0.00}
& {\fontsize{9}{10.5}\selectfont 50.7063$\pm$0.23} \\
{\fontsize{9}{10.5}\selectfont \textbf{PairFuse (Ours)}}
& {\fontsize{9}{10.5}\selectfont \textbf{2.3748}$\pm$\textbf{0.00}}
& {\fontsize{9}{10.5}\selectfont \textbf{41.61}$\pm$\textbf{0.53}}
& {\fontsize{9}{10.5}\selectfont \underline{15.6882$\pm$0.08}}
& {\fontsize{9}{10.5}\selectfont \textbf{0.7019}$\pm$\textbf{0.00}}
& {\fontsize{9}{10.5}\selectfont \textbf{52.4309}$\pm$\textbf{0.18}} \\
\bottomrule
\end{tabular*}
\endgroup
\end{table*}

\paragraph{Antibody Design.}
On RAbD, PairFuse achieves the lowest mean complex RMSD and the
highest mean DockQ and native-contact recall.
Relative to AnewOmni, complex RMSD decreases from 2.5563 to
2.3748~\AA{}, DockQ increases from 0.6853 to 0.7019, and
native-contact recall improves from 50.9592\% to 52.4309\%.
On RAbD, PairFuse increases the fraction of designs with receptor-aligned HCDR3 C$\alpha$ RMSD $\leq 2$~\AA{} from 34.00\% to 41.61\%, a gain of 7.61 percentage points (22.4\% relative).
EnzyAdapter also improves these structural and interface metrics
over AnewOmni, but by smaller margins.
LLM4GEN achieves the highest amino-acid recovery (AAR),
closely followed by PairFuse.

\begin{table*}[!t]
\centering
\caption{
Peptide design on LNR (mean $\pm$ std over 3 random seeds).
The best and second-best mean metric values are shown in \textbf{bold} and \underline{underline}, respectively.
}
\label{tab:main-lnr}
\begingroup
\fontsize{9}{10.5}\selectfont
\setlength{\tabcolsep}{4pt}
\renewcommand{\arraystretch}{1.18}
\begin{tabular*}{\textwidth}{@{\extracolsep{\fill}}lccccc@{}}
\toprule
& \multicolumn{2}{c}{Structure and Sequence}
& \multicolumn{3}{c}{Interface Quality} \\
\cmidrule(lr){2-3}
\cmidrule(l){4-6}
Method
& Complex RMSD $\downarrow$
& AAR $\uparrow$
& DockQ $\uparrow$
& Contact $\uparrow$
& lDDT $\uparrow$ \\
\midrule
{\fontsize{9}{10.5}\selectfont AnewOmni~\citep{kong2026anewomni}}
& {\fontsize{9}{10.5}\selectfont 8.5619}
& {\fontsize{9}{10.5}\selectfont 6.4306}
& {\fontsize{9}{10.5}\selectfont 0.4280}
& {\fontsize{9}{10.5}\selectfont 33.1855}
& {\fontsize{9}{10.5}\selectfont 0.4197} \\
\midrule
{\fontsize{9}{10.5}\selectfont LLM4GEN~\citep{liu2025llm4gen}}
& {\fontsize{9}{10.5}\selectfont \underline{8.5006$\pm$0.05}}
& {\fontsize{9}{10.5}\selectfont \underline{6.4805$\pm$0.04}}
& {\fontsize{9}{10.5}\selectfont \underline{0.4302$\pm$0.00}}
& {\fontsize{9}{10.5}\selectfont \underline{33.2962$\pm$0.08}}
& {\fontsize{9}{10.5}\selectfont \underline{0.4213$\pm$0.00}} \\
{\fontsize{9}{10.5}\selectfont EnzyAdapter~\citep{song2025enzycontrol}}
& {\fontsize{9}{10.5}\selectfont 8.5859$\pm$0.04}
& {\fontsize{9}{10.5}\selectfont 6.4383$\pm$0.06}
& {\fontsize{9}{10.5}\selectfont 0.4272$\pm$0.00}
& {\fontsize{9}{10.5}\selectfont 32.9352$\pm$0.03}
& {\fontsize{9}{10.5}\selectfont 0.4193$\pm$0.00} \\
{\fontsize{9}{10.5}\selectfont IP-Adapter~\citep{ye2023ipadapter}}
& {\fontsize{9}{10.5}\selectfont 8.5664$\pm$0.05}
& {\fontsize{9}{10.5}\selectfont \textbf{6.5084}$\pm$\textbf{0.00}}
& {\fontsize{9}{10.5}\selectfont 0.4248$\pm$0.00}
& {\fontsize{9}{10.5}\selectfont 32.7373$\pm$0.23}
& {\fontsize{9}{10.5}\selectfont 0.4169$\pm$0.00} \\
{\fontsize{9}{10.5}\selectfont \textbf{PairFuse (Ours)}}
& {\fontsize{9}{10.5}\selectfont \textbf{8.2821}$\pm$\textbf{0.05}}
& {\fontsize{9}{10.5}\selectfont 6.3436$\pm$0.02}
& {\fontsize{9}{10.5}\selectfont \textbf{0.4383}$\pm$\textbf{0.00}}
& {\fontsize{9}{10.5}\selectfont \textbf{34.0958}$\pm$\textbf{0.15}}
& {\fontsize{9}{10.5}\selectfont \textbf{0.4291}$\pm$\textbf{0.00}} \\
\bottomrule
\end{tabular*}
\endgroup
\end{table*}

\paragraph{Peptide Design.}
On LNR, PairFuse achieves the lowest mean complex RMSD and the
highest mean scores on all three interface-quality metrics.
Relative to AnewOmni, complex RMSD decreases from 8.5619 to
8.2821~\AA{}, and DockQ increases from 0.4280 to 0.4383.
Native-contact recall also improves from 33.1855\% to
34.0958\%, accompanied by higher backbone interface lDDT.
These improvements are accompanied by a small decrease in AAR
from 6.4306\% to 6.3436\%, with IP-Adapter achieving the
highest sequence recovery (6.5084\%).

\begin{table*}[!t]
\centering
\caption{
Small-molecule design on CrossDocked2020 (mean $\pm$ std over 3 random seeds).
The best and second-best mean metric values are shown in \textbf{bold} and \underline{underline}, respectively.
}
\label{tab:main-crossdocked2020}
\begingroup
\fontsize{9}{10.5}\selectfont
\setlength{\tabcolsep}{3pt}
\renewcommand{\arraystretch}{1.18}
\begin{tabular*}{\textwidth}{@{\extracolsep{\fill}}l@{\hspace{2pt}}c@{\hspace{2pt}}c@{\hspace{2pt}}c@{\hspace{2pt}}c@{\hspace{2pt}}c@{}}
\toprule
& \multicolumn{3}{c}{Docking}
& \multicolumn{2}{c}{Chemical Quality} \\
\cmidrule(lr){2-4}
\cmidrule(l){5-6}
Method
& Vina score $\downarrow$
& Vina min $\downarrow$
& High Affinity (\%) $\uparrow$
& QED $\uparrow$
& Atom-type MAE $\downarrow$ \\
\midrule
AnewOmni~\citep{kong2026anewomni}
& \underline{-5.1279}
& \underline{-5.9715}
& 44.8174
& 0.5745
& 0.4188 \\
\midrule
LLM4GEN~\citep{liu2025llm4gen}
& -5.0685$\pm$0.04
& -5.9298$\pm$0.04
& 43.7900$\pm$1.11
& \textbf{0.5765$\pm$0.00}
& \underline{0.4153$\pm$0.04} \\
EnzyAdapter~\citep{song2025enzycontrol}
& -5.1230$\pm$0.05
& -5.9478$\pm$0.04
& \underline{45.8629$\pm$0.92}
& \underline{0.5753$\pm$0.00}
& 0.4404$\pm$0.02 \\
IP-Adapter~\citep{ye2023ipadapter}
& -5.1040$\pm$0.03
& -5.9130$\pm$0.03
& 44.4766$\pm$1.12
& 0.5675$\pm$0.00
& 0.4551$\pm$0.05 \\
\textbf{PairFuse (Ours)}
& \textbf{-5.1517}$\pm$\textbf{0.07}
& \textbf{-5.9841}$\pm$\textbf{0.03}
& \textbf{46.8298}$\pm$\textbf{1.19}
& 0.5721$\pm$0.00
& \textbf{0.3775}$\pm$\textbf{0.03} \\
\bottomrule
\end{tabular*}
\endgroup
\end{table*}
\paragraph{Small-Molecule Design.}
On CrossDocked2020, PairFuse achieves the lowest mean Vina score
and Vina min, the highest High Affinity, and the lowest
Atom-type MAE among the compared methods.
Relative to AnewOmni, Vina score decreases from -5.1279 to
-5.1517~kcal/mol, Vina min decreases from -5.9715 to
-5.9841~kcal/mol, High Affinity increases from 44.8174\% to
46.8298\%, and Atom-type MAE decreases from 0.4188 to 0.3775.
These improvements combine more favorable docking scores
with closer agreement to the reference atom-type composition,
while maintaining a QED of 0.5721.

\paragraph{Overall Results.}
Across antibody and peptide design, PairFuse consistently
achieves the lowest mean complex RMSD and the highest mean
scores on all three interface-quality metrics.
For small-molecule design, PairFuse achieves favorable mean
docking scores with broadly comparable performance on other
chemical-quality metrics.
These results support the effectiveness of parameter-efficient
fusion of pretrained all-atom representations for unified
binder design.
Detailed implementation settings, including computational resources, training hyperparameters, checkpoint selection, and inference, are provided in \cref{app:experimental-details}.

  \subsection{Ablation Studies and Implementation Details}

We examine whether transferring pretrained pair representations directly into native attention improves interface quality beyond scalar-feature fusion.
We compare Full PairFuse with two controls.
Scalar Fusion retains scalar-feature fusion, including its pair conditioning, but removes the additive conditioning of native attention logits.
Scalar Fusion $+$ Cross-Attention Bias additionally retains the single-conditioned term in \eqref{eq:attention-fusion}, omitting only the direct pair contribution. All variants receive the same pretrained representations $(S,P)$ and use the same frozen backbone, training data, twenty-hour training budget, and validation-based checkpoint-selection rule. We evaluate training seed 42 on 93 peptide and 60 antibody targets with 30 generation attempts per target, averaging candidate scores within each target before weighting targets equally.

Table~\ref{tab:pairwise-fusion-main} shows that Full PairFuse achieves higher scores on all three interface metrics and lower complex RMSD than both controls in both tasks. Compared with Scalar Fusion $+$ Cross-Attention Bias (abbreviated as Scalar Fusion $+$ Bias in the table), it achieves DockQ of 0.4397 versus 0.4329, contact recall of 34.2175\% versus 33.5094\%, and backbone interface lDDT of 0.4309 versus 0.4233 on peptide design. The corresponding antibody scores are 0.7024 versus 0.6849, 52.4474\% versus 50.8340\%, and 0.7017 versus 0.6848, respectively. Complex RMSD decreases from 8.3991 to 8.2543~\AA{} on peptide design and from 2.5625 to 2.3704~\AA{} on antibody design. For the evaluated checkpoints, all six paired improvements in the three interface metrics pooled across peptide and antibody targets have pointwise 95\% target-bootstrap confidence intervals above zero. These results support an additional benefit of transferring pretrained pair representations directly into native attention beyond their use in scalar-feature fusion.

Additional ablations on single and pair representations and computational-cost analysis are reported in \cref{app:ablations}.

\begin{table*}[!t]
\centering
\caption{Pairwise attention fusion ablation on peptide and antibody design (training seed 42).}
\label{tab:pairwise-fusion-main}
\begingroup
\fontsize{9}{10.5}\selectfont
\setlength{\tabcolsep}{3pt}
\renewcommand{\arraystretch}{1.18}
\begin{tabular*}{\textwidth}{@{\extracolsep{\fill}}lcccccccc@{}}
\toprule
& \multicolumn{4}{c}{Peptide design (LNR)}
& \multicolumn{4}{c}{Antibody design (RAbD)} \\
\cmidrule(lr){2-5}
\cmidrule(l){6-9}
Method
& Complex RMSD $\downarrow$
& DockQ $\uparrow$
& Contact $\uparrow$
& lDDT $\uparrow$
& Complex RMSD $\downarrow$
& DockQ $\uparrow$
& Contact $\uparrow$
& lDDT $\uparrow$ \\
\midrule
Scalar Fusion
& 8.4545 & 0.4302 & 33.2426 & 0.4208
& \underline{2.5114} & \underline{0.6895} & \underline{51.5896} & \underline{0.6887} \\
Scalar Fusion $+$ Bias
& \underline{8.3991} & \underline{0.4329} & \underline{33.5094} & \underline{0.4233}
& 2.5625 & 0.6849 & 50.8340 & 0.6848 \\
\textbf{PairFuse (Ours)}
& \textbf{8.2543} & \textbf{0.4397} & \textbf{34.2175} & \textbf{0.4309}
& \textbf{2.3704} & \textbf{0.7024} & \textbf{52.4474} & \textbf{0.7017} \\
\bottomrule
\end{tabular*}
\endgroup
\end{table*}

\section{Discussion and Conclusion}
\label{sec:discussion}

In this work, we introduce \textbf{PairFuse}, a parameter-efficient framework for extending pretrained all-atom representations to unified binder design.
PairFuse integrates pretrained PairFormer single and pair representations into the evolving invariant block features and native geometric attention of a geometric latent diffusion model, while keeping both pretrained models frozen.
Theoretically, we prove that PairFuse preserves rotational equivariance and characterize the additional expressivity of pairwise attention fusion under a stated rank condition
(\cref{thm:geometry-compatible-transfer,thm:pairwise-expressivity}).
Experimentally, PairFuse achieves the lowest mean complex RMSD and the highest mean interface-quality scores on antibody and peptide benchmarks.
For small-molecule design, it achieves favorable docking and atom-type metrics, with mixed results on other chemical-quality metrics
(\cref{sec:lnr-results}).
A limitation is that unknown generated molecular structures, particularly small-molecule ligands, lack direct PairFormer token correspondences, which may affect representation-transfer effectiveness.
Future work will explore representation-fusion mechanisms that better accommodate unknown molecular structures and additional pretrained all-atom models.

\newpage
\appendix
\label{sec:append}
\part*{Appendix}
{
\setlength{\parskip}{-0em}
\startcontents[sections]
\printcontents[sections]{ }{1}{}
}
\clearpage

\section{Related Work}
\label{sec:related-work}

Our approach connects unified biomolecular generation with pretrained all-atom structural representations. We review related work in molecular design, equivariant geometric modeling, and lightweight adaptation to place this connection in context.

\paragraph{Protein and Peptide Design.}
Deep generative models have been applied to protein sequence design, structure generation, and binder optimization.
ProteinMPNN performs sequence design conditioned on protein backbones
~\citep{dauparas2022proteinmpnn}.
RFdiffusion, Chroma, and FrameDiff generate protein structures through geometric diffusion or related generative processes
~\citep{watson2023rfdiffusion,ingraham2023chroma,yim2023framediff}.
DiffAb studies antigen-conditioned antibody sequence--structure design
~\citep{luo2022diffab}.
PepGLAD and PepFlow address full-atom peptide generation through geometric latent diffusion and multimodal flow matching
~\citep{kong2024pepglad,li2024pepflow}.
Structure predictors have also been incorporated into peptide docking and binder optimization pipelines
~\citep{tsaban2022peptide,pacesa2025bindcraft}.
These advances motivate the use of pretrained structural representations for unified binder design. PairFuse integrates such representations into a geometric latent diffusion model that generates binders.

\paragraph{Unified Biomolecular Generation.}
Recent work seeks broader generative interfaces across heterogeneous molecular classes.
UniMoMo and AnewOmni represent different molecular entities through shared block-level latent states
~\citep{kong2025unimomo,kong2026anewomni}.
PocketXMol adopts atom-level representations for molecular interaction modeling, while ODesign studies broad all-atom biomolecular interaction design
~\citep{peng2024pocketxmol,zhang2025odesign}.
La-Proteina combines explicit backbone geometry with latent representations of sequence and side-chain structure
~\citep{geffner2025laproteina}.
Proteina-Complexa extends atomistic generative modeling to protein binder design with large-scale pretraining and test-time computation
~\citep{didi2026proteinacomplexa}.
PairFuse connects these approaches by incorporating pretrained all-atom representations into the shared latent space of AnewOmni through layerwise fusion.

\paragraph{All-Atom Biomolecular Models.}
Recent all-atom biomolecular models maintain token-level and pairwise representations before structure prediction or generation.
AlphaFold introduced learned representations for high-accuracy protein structure prediction
~\citep{jumper2021alphafold}.
RoseTTAFold All-Atom and AlphaFold 3 extend structural modeling across proteins, nucleic acids, ligands, and other molecular entities
~\citep{krishna2024rosettafoldallatom,abramson2024alphafold3}.
Boltz-2 further combines biomolecular structure modeling with binding-affinity prediction
~\citep{passaro2025boltz2}.
Recent generative models such as RFdiffusion3 and BoltzGen perform all-atom biomolecular generation
~\citep{butcher2025rfdiffusion3,stark2025boltzgen}.
We extract single and pair representations from the frozen BoltzGen PairFormer and integrate them into the AnewOmni denoiser. These representations provide structural information about individual molecular blocks and their interactions.

\paragraph{Equivariant Geometric Models.}
Geometric neural networks incorporate rotational and translational symmetry through equivariant attention, message passing, and vector-valued representations
~\citep{fuchs2020se3transformer,satorras2021egnn,jing2021gvp}.
The Generalist Equivariant Transformer extends these ideas to molecular interaction learning across heterogeneous chemical systems
~\citep{kong2024get}.
AnewOmni similarly combines invariant latent features with equivariant three-dimensional coordinates
~\citep{kong2026anewomni}.
PairFuse integrates all-atom representations into invariant scalar features and pairwise attention while retaining the native equivariant vector operations. We establish that this layerwise fusion preserves rotational equivariance under rotation-invariant transferred representations.

\paragraph{Lightweight Model Adaptation.}
Parameter-efficient adaptation has been studied through trainable bottleneck adapters and low-rank weight updates
~\citep{houlsby2019adapters,hu2022lora}.
Frozen pretrained models have also been connected through compact learned interfaces in multimodal systems such as Flamingo and BLIP-2
~\citep{alayrac2022flamingo,li2023blip2}.
For diffusion models, ControlNet and T2I-Adapter inject additional structural conditions into pretrained generators
~\citep{zhang2023controlnet,mou2024t2iadapter},
while IP-Adapter transfers external representations through cross-attention
~\citep{ye2023ipadapter}.
We extend this approach to unified biomolecular binder design by connecting a pretrained all-atom model with a geometric latent diffusion model. Lightweight adapters transfer PairFormer representations through layerwise scalar residuals and attention biases, allowing both pretrained backbones to remain frozen.

\clearpage

\section{Proof of Theorem~\ref{thm:geometry-compatible-transfer}}
\label{app:geometry-proof}

\begin{theorem}[Theorem~\ref{thm:geometry-compatible-transfer} Restated: Geometric Compatibility]
Suppose the native conditioning features $a$, the transferred
representations $S,P$, and the correspondence mask $m$ are
rotation invariant.
Assume that the native input transformations and output heads
preserve scalar invariance and vector equivariance, and that each
EPT layer preserves these properties for invariant attention
weights.
Then, for any fusion parameters $\phi$ and every
$R\in\mathrm{SO}(3)$, the augmented denoiser satisfies
\begin{align*}
    \epsilon_{\mathrm{aug},h}(R\cdot z_t,t|a,S,P)
    &=
    \epsilon_{\mathrm{aug},h}(z_t,t|a,S,P),
    \\
    \epsilon_{\mathrm{aug},q}(R\cdot z_t,t|a,S,P)
    &=
    R\cdot\epsilon_{\mathrm{aug},q}(z_t,t|a,S,P).
\end{align*}
Here $R\cdot z_t$ rotates generated and context coordinates
jointly, together with any native geometric prompts omitted
from the notation.
\end{theorem}

\begin{proof}
Fix a diffusion step $t$, a rotation $R\in\mathrm{SO}(3)$,
and arbitrary fusion parameters $\phi$.
We compare the denoiser evaluations at $z_t=(h_t,q_t)$ and
$R\cdot z_t=(h_t,q_tR^\top)$, using the same parameters and
rotation-invariant inputs $a,S,P,m$.
We prove by induction that the internal features satisfy
\begin{align}
    U^{(\ell)}\mapsto U^{(\ell)},
    \qquad
    \vec V^{(\ell)}\mapsto R\cdot\vec V^{(\ell)}.
    \label{eq:proof-feature-transformation}
\end{align}
Layer indices are restored for the induction, while
attention-head indices remain omitted.

\paragraph{Initialization.}
The joint rotation leaves $h_t$, $a$, and $t$ unchanged.
The native input embedding
\begin{align*}
    U^{\mathrm{in}}
    =
    \operatorname{InputEmbedding}_{\theta}(h_t,a,t)
\end{align*}
is therefore invariant.
The native geometric initialization uses the coordinates $q_t$
to construct scalar and vector features.
Pairwise distances remain unchanged under rotation, while
relative coordinates rotate consistently with the inputs.
By the assumed equivariance of the native initialization,
\begin{align*}
    \operatorname{GraphInitialization}_{\theta}
    (U^{\mathrm{in}},q_tR^\top)
    =
    (U^{(0)},R\cdot\vec V^{(0)}).
\end{align*}
Thus, \eqref{eq:proof-feature-transformation} holds at layer zero.

\paragraph{Inductive Step.}
Suppose \eqref{eq:proof-feature-transformation} holds at the
input to layer $\ell$.
Write $U=U^{(\ell)}$ and $\vec V=\vec V^{(\ell)}$ within this
layer, following \cref{sec:representation-transfer}.

First, consider scalar-feature fusion.
Since $S$ is invariant, the encoded representations
$\bar S=f_s(S)$ and the reader keys and values are invariant.
The queries depend only on the invariant scalar features $U$.
The pair bias $b(P)$ is also invariant, and the unchanged mask
$m$ selects the same available keys and pairs.
Consequently, the cross-attention weights satisfy
\begin{align*}
    A
    =
    \operatorname{MaskedSoftmax}_m
    \left(\frac{QK^\top}{\sqrt{d_h}}+b(P)\right)
    \mapsto A.
\end{align*}
Both the cross-attention output and the masked direct projection
in \eqref{eq:feature-fusion} are therefore invariant.
Concatenating heads and applying scalar output projections
preserve this property, giving
\begin{align}
    \Delta U&\mapsto\Delta U,
    \qquad
    \widetilde U=U+\Delta U\mapsto\widetilde U.
    \label{eq:proof-scalar-invariance}
\end{align}

Next, consider pairwise attention fusion.
The native logits $L(\widetilde U)$ use the enhanced scalar
features, rotation-invariant geometric biases, and the unchanged
native attention mask.
They remain invariant because $\widetilde U$ is invariant.
Each entry $\Delta L_{ij}$ in \eqref{eq:attention-fusion}
depends only on $U_i$, $S_j$, $P_{ij}$, and the coverage factors
$m_j$ and $m_i m_j$.
All these inputs are invariant, and the encoders, layer
normalization, and projections act on scalar features.
Thus, both terms in the added bias are invariant.
This argument uses the pre-fusion features $U$, as required by
the definition of $\Delta L$.
It follows that
\begin{align}
    \Delta L&\mapsto\Delta L,
    \qquad
    \widetilde L=L(\widetilde U)+\Delta L
    \mapsto\widetilde L.
    \label{eq:proof-attention-invariance}
\end{align}

The augmented attention weights
$\operatorname{Softmax}(\widetilde L)$ are therefore invariant.
The scalar attended features remain invariant since they combine
these weights with scalar values derived from $\widetilde U$.
For vector values, the native projection $W_v$ acts on feature
channels, while attention aggregates over block indices.
Both operations commute with rotation of the spatial axis:
\begin{align*}
    \operatorname{Softmax}(\widetilde L)
    \left((R\cdot\vec V)W_v\right)
    =
    R\cdot
    \left[
        \operatorname{Softmax}(\widetilde L)(\vec V W_v)
    \right].
\end{align*}
Thus, the scalar and vector attention outputs preserve their
respective transformation laws.
The argument applies to each head separately.
Native head aggregation and residual connections preserve these
properties, as do the geometric feed-forward and conditioning
updates by assumption.
Therefore,
\begin{align*}
    U^{(\ell+1)}\mapsto U^{(\ell+1)},
    \qquad
    \vec V^{(\ell+1)}\mapsto R\cdot\vec V^{(\ell+1)},
\end{align*}
which completes the inductive step.

\paragraph{Output Prediction.}
By induction, the final scalar features $U^{(L)}$ are invariant
and the vector features $\vec V^{(L)}$ are equivariant.
The native output heads preserve these properties by assumption.
Selecting the generated-block rows also commutes with rotation,
since the generated-block identities do not change.
Consequently,
\begin{align*}
    \epsilon_{\mathrm{aug},h}(R\cdot z_t,t|a,S,P)
    &=
    \epsilon_{\mathrm{aug},h}(z_t,t|a,S,P),
    \\
    \epsilon_{\mathrm{aug},q}(R\cdot z_t,t|a,S,P)
    &=
    R\cdot\epsilon_{\mathrm{aug},q}(z_t,t|a,S,P).
\end{align*}
The argument holds for arbitrary fusion parameters $\phi$,
which establishes the claimed invariance and equivariance.
\end{proof}

\clearpage

\section{Proof of Theorem~\ref{thm:pairwise-expressivity}}
\label{app:pairwise-expressivity}

\begin{theorem}[Theorem~\ref{thm:pairwise-expressivity} Restated: Expressivity of Pairwise Attention Fusion]
Fix a diffusion state, an EPT layer, and one native attention
head with query--key dimension $d_{\mathrm{att}}$.
Consider $n$ valid blocks and allow every query--key pair.
Suppose the native logits take the form
\begin{align*}
    L(U)
    =
    Q_{\mathrm{nat}}(U)
    K_{\mathrm{nat}}(U)^\top+G,
\end{align*}
where
$Q_{\mathrm{nat}}(U),K_{\mathrm{nat}}(U)
\in\mathbb R^{n\times d_{\mathrm{att}}}$
include the native normalization and scaling.
Assume that the finite geometric bias matrix
$G\in\mathbb R^{n\times n}$
remains fixed under scalar-feature updates.

A scalar-feature-only adapter may access both $S$ and $P$
and produce arbitrary updated scalar features $\widehat U$,
but adds no direct bias to the native attention logits.
Let
\begin{align*}
    \Pi_F
    &=
    \operatorname{Softmax}(L(\widehat U)),
    \\
    \Pi_{\mathrm{full}}
    &=
    \operatorname{Softmax}(\widetilde L),
\end{align*}
where $\widetilde L$ denotes the PairFuse logits in
\eqref{eq:attention-fusion}.
Both softmax operations act row-wise, and all logits are finite.
Define
\begin{align*}
    H
    =
    I_n-\frac{1}{n}\mathbf 1\mathbf 1^\top,
    \qquad
    T=(\widetilde L-G)H,
\end{align*}
where $I_n$ is the identity matrix and
$\mathbf 1\in\mathbb R^n$ is the all-ones vector.
Then every scalar-feature-only adapter satisfies
\begin{align*}
    \left\|
        (\log\Pi_{\mathrm{full}}-\log\Pi_F)H
    \right\|_F^2
    \geq
    \sum_{r>d_{\mathrm{att}}}\sigma_r(T)^2,
\end{align*}
where $\log$ acts element-wise, $\|\cdot\|_F$ denotes the
Frobenius norm, and $\sigma_r(T)$ denotes the $r$th largest
singular value of $T$.
If $\operatorname{rank}(T)>d_{\mathrm{att}}$,
no scalar-feature-only adapter reproduces
$\Pi_{\mathrm{full}}$ at this head and state.
\end{theorem}

\begin{proof}
Fix a diffusion state, an EPT layer, a native attention head,
and an arbitrary scalar-feature-only adapter.
We compare its attention weights with those of PairFuse,
using the same native projections and geometric bias $G$.
Write
\begin{align*}
    \widehat Q
    =
    Q_{\mathrm{nat}}(\widehat U),
    \qquad
    \widehat K
    =
    K_{\mathrm{nat}}(\widehat U).
\end{align*}

\paragraph{Row-Centering.}
For any finite matrix $Z\in\mathbb R^{n\times n}$,
define $\eta(Z)\in\mathbb R^n$ by
\begin{align*}
    \eta_i(Z)
    =
    \log\left(
        \sum_{j=1}^n \exp(Z_{ij})
    \right).
\end{align*}
The definition of row-wise softmax gives
\begin{align*}
    \log\operatorname{Softmax}(Z)
    =
    Z-\eta(Z)\mathbf 1^\top.
\end{align*}
Since $\mathbf 1^\top H=0$, right multiplication by $H$
removes the row-wise normalization constants:
\begin{align}
    \bigl(
        \log\operatorname{Softmax}(Z)
    \bigr)H
    =
    ZH.
    \label{eq:app-softmax-centering}
\end{align}
Applying this identity to the two attention matrices yields
\begin{align*}
    (\log\Pi_{\mathrm{full}}-G)H
    &=
    (\widetilde L-G)H=T,
    \\
    (\log\Pi_F-G)H
    &=
    \widehat Q\widehat K^\top H.
\end{align*}
Consequently,
\begin{align}
    (\log\Pi_{\mathrm{full}}-\log\Pi_F)H
    =
    T-\widehat Q\widehat K^\top H.
    \label{eq:app-centered-comparison}
\end{align}

\paragraph{Low-Rank Constraint.}
The native query and key features each have
$d_{\mathrm{att}}$ columns.
Therefore,
\begin{align*}
    \operatorname{rank}
    (\widehat Q\widehat K^\top H)
    \leq
    \operatorname{rank}
    (\widehat Q\widehat K^\top)
    \leq
    d_{\mathrm{att}}.
\end{align*}
This bound holds for arbitrary updated features $\widehat U$,
including updates that depend on both $S$ and $P$.
The constraint follows from the native query--key dimension,
not from restrictions on the scalar-feature adapter.

\paragraph{Approximation Bound.}
The Eckart--Young theorem~\citep{eckart1936approximation}
characterizes the best rank-$d_{\mathrm{att}}$ approximation
to $T$ in Frobenius norm.
Together with \eqref{eq:app-centered-comparison}, it gives
\begin{align}
    &\left\|
        (\log\Pi_{\mathrm{full}}-\log\Pi_F)H
    \right\|_F^2
    \notag\\
    &\quad=
    \left\|
        T-\widehat Q\widehat K^\top H
    \right\|_F^2
    \notag\\
    &\quad\geq
    \min_{\substack{
        Z\in\mathbb R^{n\times n}\\
        \operatorname{rank}(Z)\leq d_{\mathrm{att}}
    }}
    \|T-Z\|_F^2
    \notag\\
    &\quad=
    \sum_{r>d_{\mathrm{att}}}\sigma_r(T)^2.
    \label{eq:app-feature-only-lower-bound}
\end{align}
We interpret an empty sum as zero.
The minimization considers all matrices of rank at most
$d_{\mathrm{att}}$, including matrices that the frozen
native projections might not realize.
It therefore gives a valid lower bound for every
scalar-feature-only adapter.

If $\operatorname{rank}(T)>d_{\mathrm{att}}$,
at least one singular value in the sum is strictly positive.
Identical attention matrices would make the left-hand side
of \eqref{eq:app-feature-only-lower-bound} zero,
which gives a contradiction.
Thus, no scalar-feature-only adapter reproduces
$\Pi_{\mathrm{full}}$ under the stated rank condition.

\paragraph{Explicit Pair-Conditioned Contribution.}
We additionally establish the stronger condition for a
comparison that retains the single-conditioned attention bias
but omits the explicit pair-conditioned term.
Following \cref{eq:attention-fusion}, write
$\Delta L=C+B$, where
\begin{align*}
    C_{ij}
    &=
    \frac{m_j}{\sqrt{d_b}}
    (U_iW_{\mathrm Q}^{L})
    (\operatorname{LN}(S_j)W_{\mathrm K}^{L})^\top,
    \\
    B_{ij}
    &=
    m_i m_j f_p(P_{ij})W_{\mathrm P}^{L}.
\end{align*}
Let $D_m=\operatorname{diag}(m)$ and define
\begin{align*}
    X
    &=
    \frac{1}{\sqrt{d_b}}UW_{\mathrm Q}^{L},
    \\
    Y
    &=
    D_m\operatorname{LN}(S)W_{\mathrm K}^{L}.
\end{align*}
Both $X$ and $Y$ have $d_b$ columns, and $C=XY^\top$.
Hence,
\begin{align*}
    \operatorname{rank}(C)\leq d_b.
\end{align*}

Consider a comparison with logits
\begin{align*}
    L_{\mathrm{no\text{-}pair}}
    =
    Q_0K_0^\top+G+C_0,
\end{align*}
where
$Q_0,K_0\in\mathbb R^{n\times d_{\mathrm{att}}}$
denote its native query and key features, and $C_0$ has the
same single-conditioned factorized form with
$\operatorname{rank}(C_0)\leq d_b$.
This comparison may still use $P$ in scalar-feature fusion;
it only omits the explicit pair-conditioned logit term.
Define
\begin{align*}
    \Pi_{\mathrm{no\text{-}pair}}
    =
    \operatorname{Softmax}(L_{\mathrm{no\text{-}pair}}).
\end{align*}
Rank subadditivity gives
\begin{align*}
    \operatorname{rank}
    \bigl((Q_0K_0^\top+C_0)H\bigr)
    &\leq
    \operatorname{rank}(Q_0K_0^\top H)
    +
    \operatorname{rank}(C_0H)
    \\
    &\leq
    d_{\mathrm{att}}+d_b.
\end{align*}
Applying \eqref{eq:app-softmax-centering} and the
Eckart--Young theorem again, we obtain
\begin{align}
    &\left\|
        (\log\Pi_{\mathrm{full}}
        -
        \log\Pi_{\mathrm{no\text{-}pair}})H
    \right\|_F^2
    \notag\\
    &\quad=
    \left\|
        T-(Q_0K_0^\top+C_0)H
    \right\|_F^2
    \notag\\
    &\quad\geq
    \sum_{r>d_{\mathrm{att}}+d_b}
    \sigma_r(T)^2.
    \label{eq:app-explicit-pair-lower-bound}
\end{align}
Therefore, if
$\operatorname{rank}(T)>d_{\mathrm{att}}+d_b$,
even this stronger comparison cannot reproduce
$\Pi_{\mathrm{full}}$.
This establishes the additional expressivity condition
for the explicit pair-conditioned attention bias.
\end{proof}

\clearpage

\section{Experimental Details}
\label{app:experimental-details}

Our implementation combines frozen pretrained representations with trainable layerwise fusion. We describe the construction of these representations, the training and model-selection procedure, and the inference and evaluation protocol used across the three design tasks.

\subsection{Training, Validation, and Test Data}
\label{app:data-splits}
We construct the adaptation data from the existing AnewOmni-format dataset partitions~\citep{kong2026anewomni}. Antibody examples are drawn from the SAbDab training partition~\citep{dunbar2014sabdab}, peptide examples from the ProtFrag and PepBench training partitions~\citep{kong2024pepglad}, and small-molecule examples from the PDBbind v2020 training partition~\citep{liu2015pdbbind}. The validation set retains the corresponding source validation records: 400 antibody examples, 114 peptide examples, and 100 small-molecule examples, for a total of 614 records.

\paragraph{Training-Set Construction.}
Before sampling, we collect PDB identifiers from the source validation and test partitions across all three modalities, including the LNR peptide benchmark, and exclude matching entries from the training pools. Peptide preprocessing additionally excludes exact receptor--peptide sequence pairs present in its validation or LNR sets. We then freeze 10,000 logical training positions per modality, yielding 30,000 positions in total. Sampling uses eligible training records, with replacement when needed to reach the prescribed count. The peptide manifest contains 9,447 ProtFrag positions and 553 PepBench positions. For each antibody position, we randomly select one valid CDR and retain that identity throughout training; repeated complexes preferentially use previously unselected valid CDRs. The resulting training and validation manifests have no shared PDB identifiers, including across modalities, and remain fixed throughout optimization.

\paragraph{Validation and Test Evaluation.}
Validation data are used to select a single shared checkpoint for all three modalities according to the native denoising-loss criterion described below. Test evaluation uses the fixed benchmark collections of 60 RAbD antibody targets~\citep{adolfbryfogle2018rabd}, 93 LNR peptide targets~\citep{tsaban2022peptide,kong2024pepglad}, and 100 CrossDocked2020 small-molecule targets~\citep{francoeur2020crossdocked}, totaling 253 targets. These benchmark collections are assembled separately from the source training and validation partitions. Each model is evaluated with 30 generation attempts per target under the common inference protocol.

\subsection{Construction of Pretrained Donor Representations}
\label{app:donor-construction}
We extract single and pair representations from the frozen BoltzGen PairFormer to provide structural information about individual molecular blocks and their interactions. The input contains the visible molecular context and the prescribed design positions, with the identity and coordinates of the generated region masked. The resulting representations are aligned with AnewOmni's molecular blocks and reused throughout training and denoising.

\paragraph{Molecular Context and Design Positions.}
Visible protein residues retain their amino-acid identities and atomic coordinates. For peptide design, the prescribed peptide length determines the generated positions. For antibody design, the designated CDR determines these positions, while the retained framework and antigen provide structural context. Protein design positions are represented by backbone placeholders with zero coordinates and no structural visibility. Their residue identities are replaced by unknown tokens before feature extraction, so neither the reference binder sequence nor its geometry is available to the donor. Protein residues are tokenized at the residue level using the native all-atom design featurization. We use single-sequence inputs without external multiple-sequence alignments or retrieved structural templates; the visible context structure enters through the native structural-conditioning pathway.

\paragraph{Representation Extraction.}
The frozen donor processes each input with three recycling steps. We retain the single and pair representations from the final PairFormer output, before structure generation:
\begin{equation}
 S^{\mathrm{tok}}\in\mathbb{R}^{N_{\mathrm{tok}}\times384},\qquad
 P^{\mathrm{tok}}\in\mathbb{R}^{N_{\mathrm{tok}}\times N_{\mathrm{tok}}\times128},
\end{equation}
where $N_{\mathrm{tok}}$ denotes the number of donor tokens. These features encode token-level structural information and pairwise relationships without requiring a sampled BoltzGen structure. Feature extraction is performed in evaluation mode without gradient computation, using bfloat16 mixed precision. Single and pair representations are retained in float32 and bfloat16, respectively, and projected by the trainable fusion encoders during adaptation.

\paragraph{Correspondence with Molecular Blocks.}
We align donor tokens with AnewOmni blocks using chain identifiers and residue correspondence. Blocks derived from the same protein residue share the corresponding donor token, and chain discontinuities are preserved. Let $\pi(i)$ denote the donor token associated with block $i$, and let $m_i\in\{0,1\}$ indicate whether this correspondence is available. The aligned features are
\begin{equation}
 S_i=S^{\mathrm{tok}}_{\pi(i)},\qquad
 P_{ij}=P^{\mathrm{tok}}_{\pi(i),\pi(j)},
 \label{eq:donor-alignment}
\end{equation}
for covered blocks and block pairs. Unavailable features are excluded through the correspondence masks in the fusion equations. This alignment preserves the distinction between donor tokens and the block-level latent representation used by the geometric denoiser.

\paragraph{Small-Molecule Coverage.}
In de novo small-molecule design, the ligand identity and chemical graph are unknown. We therefore exclude the reference ligand from the donor input and set $m_i=0$ for generated ligand blocks. Nonprotein context without a validated token correspondence is treated similarly. The transferred representations describe the available protein context rather than the unknown ligand. Generated ligand blocks can nevertheless attend to covered context tokens through scalar cross-attention and the key-masked attention bias. They do not receive a directly aligned single representation or a pair representation requiring both endpoints to be covered. This restriction limits the ligand-specific information available for transfer and helps characterize the distinction between protein-binder and small-molecule design.

\paragraph{Reuse across Denoising Steps.}
The donor representations are computed once for each fixed molecular input and reused across training iterations and reverse-diffusion steps. They are independent of the current noisy latent state, whereas the fusion queries depend on the evolving EPT hidden features. The same masked representations are used when the native training condition sampler exposes additional generated-region information to AnewOmni. Donor features are not recomputed for each sampled condition. Both pretrained models remain frozen, and only the fusion parameters are optimized.

\subsection{Training and Model Selection}
\label{app:training-recipe}
We train the fusion adapters on 10,000 fixed records per modality, with antibody CDR identities fixed throughout training. Both pretrained backbones remain frozen. Each PairFuse and baseline training run uses a twenty hour allocation on four NVIDIA A100, including initialization, training, and validation. We select the checkpoint with the lowest combined native validation total loss on the 614-record validation set. This criterion combines latent-feature and coordinate denoising losses using the native loss weighting and weights modality-level means by their numbers of generated blocks. Ties retain the earlier checkpoint. The selected checkpoint is shared across antibody, peptide, and small-molecule design.

\begin{table}[t]
\centering
\caption{Trainable parameter counts for the compared methods.}
\label{tab:trainable-parameters}
\begingroup
\small
\setlength{\tabcolsep}{12pt}
\renewcommand{\arraystretch}{1.15}
\begin{tabular}{lr}
\toprule
Method & Trainable parameters \\
\midrule
LLM4GEN & 4,028,016 \\
EnzyAdapter & 4,027,715 \\
IP-Adapter & 4,031,232 \\
\midrule
PairFuse (Ours) & 4,027,712 \\
\bottomrule
\end{tabular}
\endgroup
\end{table}

\subsection{Hyperparameters and Optimization}
\label{app:hyperparameters}
The number of fusion layers is set to six, with one fusion module at each EPT layer.
We use AdamW with a learning rate of $3\times10^{-4}$, weight decay of 0.01, an effective batch size of 32, and gradient clipping at 1. The learning rate follows a 100-update warmup and a cosine schedule over the training allocation, with a minimum value of 20\% of the initial learning rate. PairFuse uses training seeds 20261007, 42, and 123, with checkpoints selected at updates 2,838, 2,732, and 2,500, respectively. IP-Adapter and EnzyAdapter use training seeds 1, 123, and 42. LLM4GEN uses three training seeds: 1, 123123, and 42. The pathway controls use training seed 42. Training uses modality-specific microbatches under the native batch-cost limit.

\subsection{Inference and Evaluation}
\label{app:inference-protocol}
We evaluate 60 RAbD targets, 93 LNR targets, and 100 CrossDocked2020 targets. Each model receives 30 generation attempts per target, using 100 diffusion steps, and batch size 10.

For antibody and peptide design, we report structural RMSD, amino-acid recovery, steric clashes, Rosetta energies, and interface quality. Interface metrics comprise local-interface DockQ, interface RMSD, native contact recall, and backbone interface lDDT. Native metrics pool finite candidates within each run; interface metrics average candidates within each target and then weight targets equally. For small-molecule design, we report Vina scores, QED, normalized synthetic accessibility, and Lipinski rule compliance, pooling candidate scores. The original structural, energy, and chemical metrics follow the definitions used in AnewOmni~\citep{kong2026anewomni}.

For the four interface metrics, residues are matched by their benchmark identities and design-block correspondence. Native-contact recall is the percentage of reference interfacial residue pairs with a heavy-atom distance below 5~\AA{} that are recovered in the generated complex. Interface RMSD measures backbone deviation after superposing the native interface residues, defined by a 10~\AA{} interchain distance threshold. DockQ combines native-contact recovery, interface RMSD, and receptor-aligned binder RMSD using the standard DockQ normalization. Backbone interface lDDT averages distance-recovery fractions at 0.5, 1, 2, and 4~\AA{} tolerances over reference target--binder backbone-atom pairs within 15~\AA{}. Missing generated atoms count as unrecovered pairs.

\paragraph{Small-Molecule Main-Table Endpoints.}
Vina score evaluates the generated pose, and Vina min evaluates it after local minimization; both are measured in kcal/mol. High Affinity follows the better-than-reference docking criterion of TargetDiff~\citep{guan2023targetdiff}: for each target, we average $\mathbf{1}[E_{\mathrm{dock}}^{\mathrm{gen}}<E_{\mathrm{dock}}^{\mathrm{ref}}]$ over candidates with finite generated and reference Vina-dock scores, then average target percentages within each run and combine runs using the main-results aggregation.

Atom-type MAE follows the CBGBench mean-count definition~\citep{lin2024cbgbench}. Within each run, let $\bar n_a$ be the mean count of element $a$ per molecule with a finite substructure record, and $r_a$ its fixed CBGBench reference mean. We report
\[
  \mathrm{MAE}_{\mathrm{atom}}=\frac{1}{7}\sum_{a\in\{\mathrm{C,N,O,F,P,S,Cl}\}}|\bar n_a-r_a|,
\]
and combine the run-specific estimates using the main-results aggregation. Lower MAE indicates closer agreement with the reference atom-type composition. QED is the quantitative estimate of drug-likeness and ranges from 0 to 1. The complete comparison also reports Vina dock, normalized synthetic accessibility and 2D Diversity.

\section{Ablations and Computational Analysis}\label{app:ablations}
We examine how single and pair representations contribute to generation quality and quantify the computational overhead of layerwise fusion. The efficiency analysis compares inference runtime and memory usage with the pretrained baseline.

\subsection{Single and Pair Representations}
\label{sec:ablation-study}
We examine the contributions of single and pair representations by comparing single-only, pair-only, and joint fusion in the same multilayer interface. Tables~\ref{tab:ablation-rabd}--\ref{tab:ablation-crossdocked2020} report results across all three modalities with a fixed 20261007 random seed. Joint fusion improves antibody complex RMSD from 2.4263 to 2.3769~\AA{} and DockQ from 0.6980 to 0.7025. For peptide design, it reduces complex RMSD from 8.4471 to 8.2497~\AA{} and increases DockQ from 0.4324 to 0.4393. For small molecules, pair-only fusion yields the highest QED (0.5801). These results show that single representations complement pair representations in structural and interface prediction, with different effects across ligand properties.
\begingroup
\begingroup
\setlength{\intextsep}{6pt}
\begin{table}[H]\centering
\setlength{\abovecaptionskip}{4pt}\setlength{\belowcaptionskip}{4pt}
\caption{Ablation of transferred representations: Antibody (RAbD; HCDR3).}\label{tab:ablation-rabd}
\label{tab:ablation}
\fontsize{10}{12}\selectfont
\setlength{\tabcolsep}{3pt}
\renewcommand{\arraystretch}{1.18}
\begin{tabular*}{\linewidth}{@{\extracolsep{\fill}}lccccc@{}}
\toprule
 & \multicolumn{2}{c}{Structure and Sequence} & \multicolumn{3}{c}{Interface Quality}\\
\cmidrule(lr){2-3}\cmidrule(lr){4-6}
Method & \mbox{Complex RMSD (\AA{}) $\downarrow$} & AAR (\%) $\uparrow$ & DockQ $\uparrow$ & Contact (\%) $\uparrow$ & lDDT $\uparrow$\\
\midrule
Matched adapter & 2.5864 & 15.2559 & 0.6861 & 51.0801 & 0.6855\\
Single only & 2.5415 & \underline{15.6205} & 0.6858 & 50.9525 & 0.6870\\
Pair only & \underline{2.4263} & 15.2193 & \underline{0.6980} & \underline{52.2595} & \underline{0.6982}\\
Full PairFuse & \textbf{2.3769} & \textbf{15.7829} & \textbf{0.7025} & \textbf{52.5997} & \textbf{0.7035}\\
\bottomrule
\end{tabular*}
\end{table}
\endgroup\endgroup
\begingroup
\begingroup
\setlength{\intextsep}{6pt}
\begin{table}[H]\centering
\setlength{\abovecaptionskip}{4pt}\setlength{\belowcaptionskip}{4pt}
\caption{Ablation of transferred representations: Peptide (LNR).}\label{tab:ablation-lnr}
\fontsize{10}{12}\selectfont
\setlength{\tabcolsep}{3pt}
\renewcommand{\arraystretch}{1.18}
\begin{tabular*}{\linewidth}{@{\extracolsep{\fill}}lccccc@{}}
\toprule
 & \multicolumn{2}{c}{Structure and Sequence} & \multicolumn{3}{c}{Interface Quality}\\
\cmidrule(lr){2-3}\cmidrule(lr){4-6}
Method & \mbox{Complex RMSD (\AA{}) $\downarrow$} & AAR (\%) $\uparrow$ & DockQ $\uparrow$ & Contact (\%) $\uparrow$ & lDDT $\uparrow$\\
\midrule
Matched adapter & 8.5092 & \textbf{6.4294} & 0.4297 & 33.2400 & 0.4206\\
Single only & \underline{8.3431} & 6.3114 & \underline{0.4356} & \underline{33.7244} & \underline{0.4256}\\
Pair only & 8.4471 & 6.3370 & 0.4324 & 33.4552 & 0.4241\\
Full PairFuse & \textbf{8.2497} & \underline{6.3655} & \textbf{0.4393} & \textbf{34.1400} & \textbf{0.4293}\\
\bottomrule
\end{tabular*}
\end{table}
\endgroup\endgroup
\begingroup
\begingroup
\setlength{\intextsep}{6pt}
\begin{table}[H]\centering
\setlength{\abovecaptionskip}{4pt}\setlength{\belowcaptionskip}{4pt}
\caption{Ablation of transferred representations: Small molecule (CrossDocked2020).}\label{tab:ablation-crossdocked2020}
\fontsize{10}{12}\selectfont\setlength{\tabcolsep}{1.5pt}\renewcommand{\arraystretch}{1.12}
\begin{tabular*}{\linewidth}{@{\extracolsep{\fill}}lcccc@{}}\toprule
 & \multicolumn{3}{c}{{\fontsize{10}{11.5}\selectfont Vina (kcal/mol)}} & \multicolumn{1}{c}{{\fontsize{10}{11.5}\selectfont Chemical Quality}}\\
\cmidrule(lr){2-4}\cmidrule(lr){5-5}
{\fontsize{10}{11.5}\selectfont Method} & {\fontsize{10}{11.5}\selectfont Score $\downarrow$} & {\fontsize{10}{11.5}\selectfont Min $\downarrow$} & {\fontsize{10}{11.5}\selectfont Dock $\downarrow$} & {\fontsize{10}{11.5}\selectfont QED $\uparrow$}\\\midrule
Matched adapter & -5.1011 & \underline{-6.0286} & -7.0355 & \underline{0.5766}\\
Single only & -5.0512 & -5.9377 & -7.0351 & 0.5752\\
Pair only & \underline{-5.1495} & -5.9830 & \underline{-7.1045} & \textbf{0.5801}\\
Full PairFuse & \textbf{-5.2283} & \textbf{-6.0295} & \textbf{-7.1400} & 0.5755\\
\bottomrule\end{tabular*}
\end{table}
\endgroup\endgroup

\paragraph{Effect of Single Representations.}
Adding single representations to pair-only fusion improves all five displayed metrics on both antibody and peptide benchmarks, including all three interface-quality metrics. For antibodies, the gains extend from complex RMSD and DockQ to Rosetta dG, which decreases from $-2.2303$ to $-8.1333$~REU. For peptides, interface RMSD decreases from 3.0615 to 2.9919~\AA{} and contact recall increases from 33.4552\% to 34.1400\%, while pair-only fusion gives the lower intra-region clash rate (0.8165\% versus 0.8982\%). These results support using both representations to improve protein-binder geometry and interface recovery.

For ligand design, pair-only fusion yields a higher QED than joint fusion (0.5801 versus 0.5755). Together, these results show that single representations complement pair representations more consistently in antibody and peptide design, while introducing a trade-off between docking performance and chemical properties in ligand design.

\paragraph{Comparison with Native-Feature Adaptation.}
The Matched adapter row reports the performance of an equally sized adapter trained under identical data and compute budgets without pretrained donor representations. On both antibody and peptide design, Full PairFuse attains lower complex RMSD and higher DockQ, contact recall, and backbone interface lDDT than this reference. For small-molecule design, the matched adapter records Vina score, Vina min, and Vina dock values of $-5.1011$, $-6.0286$, and $-7.0355$~kcal/mol, together with a QED of 0.5766. This result establishes that transferring pretrained structural knowledge provides benefits that cannot be explained by adding a similarly sized trainable adapter.
 
\subsection{Computational Cost}
\label{sec:efficiency}
Table~\ref{tab:cost} reports runtime and peak allocated memory on one NVIDIA RTX A6000 (48 GB), using one fixed input per modality, batch size ten, and 100 diffusion steps. We report the mean and sample standard deviation of five synchronized measurements after one warm-up, and the largest peak PyTorch allocated memory across those repeats. Timing includes sampling and all-atom decoding, but excludes donor-feature extraction, projection-cache preparation, data loading, and file export. Only one model is resident at a time; PairFuse's allocated memory includes resident donor features and projection caches. PairFuse takes 10.63, 6.14, and 5.07 seconds per batch for antibody, peptide, and small-molecule inputs, respectively, compared with 9.44, 4.84, and 3.65 seconds for AnewOmni. The runtime overheads are 12.6\%, 27.0\%, and 39.2\%, with peak allocated memory increasing by approximately 0.19, 0.12, and 0.17 GB.

\begingroup
\setlength{\intextsep}{6pt}
\begin{table}[H]\centering
\setlength{\abovecaptionskip}{4pt}\setlength{\belowcaptionskip}{4pt}
\caption{Inference cost on one NVIDIA RTX A6000 (48 GB). Time is mean $\pm$ sample standard deviation per batch of ten candidates over five measurements after warm-up; memory is peak PyTorch allocated GPU memory. Bold and underline indicate the best and second-best displayed means.}\label{tab:cost}
\fontsize{10}{12}\selectfont\setlength{\tabcolsep}{2pt}\renewcommand{\arraystretch}{1.12}
\begin{tabular*}{\linewidth}{@{\extracolsep{\fill}}lcccccc@{}}\toprule
 & \multicolumn{2}{c}{{\fontsize{10}{11.5}\selectfont Antibody}} & \multicolumn{2}{c}{{\fontsize{10}{11.5}\selectfont Peptide}} & \multicolumn{2}{c}{{\fontsize{10}{11.5}\selectfont Small molecule}}\\
\cmidrule(lr){2-3}\cmidrule(lr){4-5}\cmidrule(lr){6-7}
{\fontsize{10}{11.5}\selectfont Method} & {\fontsize{10}{11.5}\selectfont Time (s) $\downarrow$} & {\fontsize{10}{11.5}\selectfont \mbox{Memory (GB) $\downarrow$}} & {\fontsize{10}{11.5}\selectfont Time (s) $\downarrow$} & {\fontsize{10}{11.5}\selectfont \mbox{Memory (GB) $\downarrow$}} & {\fontsize{10}{11.5}\selectfont Time (s) $\downarrow$} & {\fontsize{10}{11.5}\selectfont \mbox{Memory (GB) $\downarrow$}}\\\midrule
AnewOmni & $\textbf{9.44}\pm0.14$ & \textbf{12.90} & $\textbf{4.84}\pm0.03$ & \textbf{6.86} & $\textbf{3.65}\pm0.11$ & \textbf{4.29}\\
PairFuse (Ours) & $\underline{10.63}\pm0.09$ & \underline{13.09} & $\underline{6.14}\pm0.14$ & \underline{6.97} & $\underline{5.07}\pm0.15$ & \underline{4.47}\\
\bottomrule\end{tabular*}
\end{table}
\endgroup

\section{Algorithms}
\label{app:layerwise-fusion}

\Cref{alg:native-ept} summarizes the native EPT computation used in
AnewOmni~\citep{jiao2026equivariant,kong2026anewomni}.
\Cref{alg:layerwise-fusion} shows our layerwise modifications in blue,
and \cref{alg:fusion-modules} details the two fusion modules.

\providecommand{\eptkw}[1]{\textcolor{teal!80!black}{\textbf{#1}}}
\providecommand{\eptcomment}[1]{\textcolor{red!80!magenta}{\textit{\#~#1}}}
\providecommand{\fusionstep}[1]{\textcolor{blue!75!black}{#1}}

\subsection{Native EPT Architecture}
\label{app:native-ept}

\begin{algorithm}[H]
\caption{Architecture of the Native EPT}
\label{alg:native-ept}
\small
\begin{algorithmic}[1]

\State \eptkw{def} $\operatorname{GraphInitialization}(U,q)$:
\Statex \eptcomment{Initialize invariant and equivariant features}

\State $e_{ij}\gets
[u_i,u_j,\operatorname{RBF}(\|q_i-q_j\|_2)]$

\State $u_i\gets
\varphi_u\left(
u_i,\sum_{j\in\mathcal N(i)}
\varphi_s(e_{ij})\odot u_j
\right)$

\State $\vec v_i\gets
\operatorname{Mean}_{j\in\mathcal N(i)}
\varphi_v(e_{ij})\otimes(q_i-q_j)$

\State \eptkw{return} $U,\vec V$

\Statex
\State \eptkw{def}
$\operatorname{EquivariantSelfAttention}(U,\vec V,q)$:

\State $Q_k,K_k\gets UW_k^Q,UW_k^K$

\State $V_k\gets[UW_k^{V_u},\vec V W_k^{V_v}]$

\State $D_{ij}\gets\|q_i-q_j\|_2$

\State $G\gets\varphi_r(\operatorname{RBF}(D))$

\State $(U_k,\vec V_k)\gets
\operatorname{Softmax}
\left(
\frac{Q_kK_k^\top}{2\sqrt{d_h}}
-D+G
\right)V_k$

\State $U,\vec V\gets
\sum_k U_kW_k^{O_u},
\sum_k\vec V_kW_k^{O_v}$

\State \eptkw{return} $U,\vec V$

\Statex
\State \eptkw{def} $\operatorname{EquivariantFFN}(U,\vec V)$:

\State $\vec V_1,\vec V_2\gets
\operatorname{VectorProjection}(\vec V)$

\State $U,g\gets
\varphi_{\mathrm{FFN}}(U,\|\vec V_1\|_2)$

\State $\vec V\gets
\operatorname{LN}(g)\odot\vec V_2$

\State \eptkw{return} $U,\vec V$

\Statex
\State \eptkw{def} $\operatorname{EPT}(U^{\mathrm{in}},q_t)$:
\Statex \eptcomment{Apply the native EPT layers}

\State $U^{(0)},\vec V^{(0)}
\gets\operatorname{GraphInitialization}(U^{\mathrm{in}},q_t)$

\For{$\ell=0,\ldots,L-1$}

    \State $(U',\vec V')\gets
    (U^{(\ell)},\vec V^{(\ell)})
    +\operatorname{EquivariantSelfAttention}
    \left(
    \operatorname{LN}(U^{(\ell)}),
    \vec V^{(\ell)},q_t
    \right)$

    \State $(U'',\vec V'')\gets
    (U',\vec V')
    +\operatorname{EquivariantFFN}
    (\operatorname{LN}(U'),\vec V')$

    \State $(U^{(\ell+1)},\vec V^{(\ell+1)})
    \gets\operatorname{NativeConditioning}_{\theta}^{(\ell)}
    (U'',\vec V'')$

\EndFor

\State \eptkw{return}
$\operatorname{LN}(U^{(L)}),\vec V^{(L)}$

\end{algorithmic}
\end{algorithm}

\subsection{Architecture of the PairFormer-Augmented EPT}

We retain the native graph initialization, feed-forward layers,
and conditioning operations in \cref{alg:native-ept}.
The only additions are scalar-feature fusion and an attention-logit
bias, highlighted in blue in \cref{alg:layerwise-fusion}.

\begin{algorithm}[H]
\caption{Architecture of the PairFormer-Augmented EPT}
\label{alg:layerwise-fusion}
\small
\begin{algorithmic}[1]

\State \eptkw{def}
$\operatorname{EquivariantSelfAttention}
(U,\vec V,q;\fusionstep{\Delta L})$:

\State $Q_k,K_k\gets UW_k^Q,UW_k^K$

\State $V_k\gets[UW_k^{V_u},\vec V W_k^{V_v}]$

\State $D_{ij}\gets\|q_i-q_j\|_2$

\State $G\gets\varphi_r(\operatorname{RBF}(D))$

\State $(U_k,\vec V_k)\gets
\operatorname{Softmax}
\left(
\frac{Q_kK_k^\top}{2\sqrt{d_h}}
-D+G
\fusionstep{+\Delta L_k}
\right)V_k$

\State $U,\vec V\gets
\sum_k U_kW_k^{O_u},
\sum_k\vec V_kW_k^{O_v}$

\State \eptkw{return} $U,\vec V$

\Statex
\State \eptkw{def}
$\operatorname{PairFormerAugmentedEPT}
(U^{\mathrm{in}},q_t,S,P,m)$:

\State $U^{(0)},\vec V^{(0)}
\gets\operatorname{GraphInitialization}(U^{\mathrm{in}},q_t)$

\State \fusionstep{$\bar S\gets f_s(S)$}

\For{$\ell=0,\ldots,L-1$}

    \State \fusionstep{
    $\Delta U^{(\ell)}
    \gets\operatorname{ScalarFeatureFusion}
    (U^{(\ell)},\bar S,P;m)$}

    \State \fusionstep{
    $\widetilde U^{(\ell)}
    \gets U^{(\ell)}+\Delta U^{(\ell)}$}

    \State \fusionstep{
    $\Delta L^{(\ell)}
    \gets\operatorname{PairwiseAttentionBias}
    (U^{(\ell)},S,P;m)$}

    \State $(U',\vec V')\gets
    (\widetilde U^{(\ell)},\vec V^{(\ell)})
    +\operatorname{EquivariantSelfAttention}
    \left(
    \operatorname{LN}(\widetilde U^{(\ell)}),
    \vec V^{(\ell)},q_t;
    \fusionstep{\Delta L^{(\ell)}}
    \right)$

    \State $(U'',\vec V'')\gets
    (U',\vec V')
    +\operatorname{EquivariantFFN}
    (\operatorname{LN}(U'),\vec V')$

    \State $(U^{(\ell+1)},\vec V^{(\ell+1)})
    \gets\operatorname{NativeConditioning}_{\theta}^{(\ell)}
    (U'',\vec V'')$

\EndFor

\State \eptkw{return}
$\operatorname{LN}(U^{(L)}),\vec V^{(L)}$

\end{algorithmic}
\end{algorithm}

\subsection{PairFormer Fusion Modules}

We use the aligned block indexing from
\cref{sec:representation-transfer}.
The coverage mask $m_i$ indicates whether block $i$ has an
available PairFormer representation.
The masked softmax attends only to keys $j$ with $m_j=1$,
and pair-dependent terms vanish unless $m_i m_j=1$.
The projections below are layer-specific, with head index $k$;
the encoders $f_s$ and $f_p$ are shared across layers.

\begin{algorithm}[H]
\caption{PairFormer Fusion Modules}
\label{alg:fusion-modules}
\small
\begin{algorithmic}[1]

\State \eptkw{def}
$\operatorname{ScalarFeatureFusion}(U,\bar S,P;m)$:

\State \fusionstep{
$Q_k,K_k,V_k\gets
UW_{Q,k},\bar S W_{K,k},\bar S W_{V,k}$}

\State \fusionstep{
$A_k\gets
\operatorname{MaskedSoftmax}_{m}
\left(
\frac{Q_kK_k^\top}{\sqrt{d_h}}+b_k(P)
\right)$}

\State \fusionstep{
$\Delta U\gets
\operatorname{Concat}_k(A_kV_k)W_O
+\operatorname{diag}(m)(\bar S W_D)$}

\State \eptkw{return} $\Delta U$

\Statex
\State \eptkw{def}
$\operatorname{PairwiseAttentionBias}(U,S,P;m)$:

\State \fusionstep{
$\Delta L_{ij,k}\gets
\frac{m_j}{\sqrt{d_b}}
(U_iW_{Q,k}^{L})
(\operatorname{LN}(S_j)W_{K,k}^{L})^\top$}

\State \fusionstep{
$\Delta L_{ij,k}\gets
\Delta L_{ij,k}
+m_i m_j f_p(P_{ij})W_{P,k}^{L}$}

\State \eptkw{return} $\Delta L$

\end{algorithmic}
\end{algorithm}

\clearpage
\section{Complete Benchmark Metrics}
\label{app:four-model-metrics}
The complete metrics follow the model and aggregation settings in \cref{sec:trimodal-update,app:hyperparameters,app:inference-protocol}. 
Best means are bold and second-best means underlined only when the metric has a specified preference direction.
\par\ifdim\dimexpr\pagegoal-\pagetotal\relax<120pt\newpage\fi\subsection{Antibody design on RAbD}
\begingroup\fontsize{8}{9.5}\selectfont\setlength{\tabcolsep}{2.5pt}\renewcommand{\arraystretch}{1.08}
\setlength{\LTpre}{12pt}\setlength{\LTpost}{6pt}
\begin{longtable}{@{}>{\raggedright\arraybackslash}p{.30\linewidth}*{4}{>{\centering\arraybackslash}p{.11\linewidth}}>{\centering\arraybackslash}p{.20\linewidth}@{}}
\caption{Antibody design on RAbD: complete metrics.}\label{tab:full-rabd}\\\toprule
{\fontsize{9}{10.5}\selectfont Metric} & {\fontsize{9}{10.5}\selectfont AnewOmni} & {\fontsize{9}{10.5}\selectfont LLM4GEN} & {\fontsize{9}{10.5}\selectfont EnzyAdapter} & {\fontsize{9}{10.5}\selectfont IP-Adapter} & {\fontsize{9}{10.5}\selectfont \mbox{PairFuse (Ours)}}\\\midrule\endfirsthead
\multicolumn{6}{c}{Antibody design on RAbD (continued)}\\\toprule
{\fontsize{9}{10.5}\selectfont Metric} & {\fontsize{9}{10.5}\selectfont AnewOmni} & {\fontsize{9}{10.5}\selectfont LLM4GEN} & {\fontsize{9}{10.5}\selectfont EnzyAdapter} & {\fontsize{9}{10.5}\selectfont IP-Adapter} & {\fontsize{9}{10.5}\selectfont \mbox{PairFuse (Ours)}}\\\midrule\endhead
\midrule\multicolumn{6}{r}{Continued on next page}\\\endfoot\bottomrule\endlastfoot
Complex RMSD [\AA{}] $\downarrow$ & 2.5563 & 2.5844 & \underline{2.5055} & 2.5416 & \textbf{2.3748}\\
AAR [\%] $\uparrow$ & 15.6755 & \textbf{15.7987} & 15.6484 & 15.4042 & \underline{15.6882}\\
Intra-region clash [\%] $\downarrow$ & 0.7501 & 0.7917 & \textbf{0.6406} & 0.6542 & \underline{0.6454}\\
Interface clash [\%] $\downarrow$ & \textbf{0.2865} & 0.3648 & 0.3569 & 0.3502 & \underline{0.3107}\\
DockQ [0--1] $\uparrow$ & 0.6853 & 0.6837 & \underline{0.6888} & 0.6854 & \textbf{0.7019}\\
iRMSD [\AA{}] $\downarrow$ & 1.1905 & 1.2002 & \underline{1.1723} & 1.1865 & \textbf{1.1180}\\
Native-contact recall [\%] $\uparrow$ & 50.9592 & 50.9243 & \underline{51.0443} & 50.7063 & \textbf{52.4309}\\
Backbone interface lDDT [0--1] $\uparrow$ & 0.6863 & 0.6843 & \underline{0.6903} & 0.6870 & \textbf{0.7026}\\
structural/\allowbreak common\_\allowbreak region\_\allowbreak CA  & 11.0000 & 11.0000 & 11.0000 & 11.0000 & 11.0000\\
structural/\allowbreak common\_\allowbreak target\_\allowbreak CA  & 220.6500 & 220.6500 & 220.6500 & 220.6500 & 220.6500\\
structural/\allowbreak region\_\allowbreak CA\_\allowbreak count  & 11.0000 & 11.0000 & 11.0000 & 11.0000 & 11.0000\\
structural/\allowbreak region\_\allowbreak block\_\allowbreak count  & 11.0000 & 11.0000 & 11.0000 & 11.0000 & 11.0000\\
Contact F1 [0--1] $\uparrow$ & 0.5237 & 0.5220 & \underline{0.5247} & 0.5232 & \textbf{0.5402}\\
LRMSD [\AA{}] $\downarrow$ & 2.4984 & 2.5187 & \underline{2.4546} & 2.4888 & \textbf{2.3387}\\
clashes $\downarrow$ & \underline{0.1096} & \textbf{0.1090} & 0.1234 & 0.1139 & 0.1166\\
complex\_\allowbreak CA\_\allowbreak RMSD $\downarrow$ & 2.5563 & 2.5844 & \underline{2.5055} & 2.5416 & \textbf{2.3748}\\
complex\_\allowbreak CA\_\allowbreak RMSD\_\allowbreak le\_\allowbreak 2A $\uparrow$ & 0.3400 & 0.3375 & \underline{0.3680} & 0.3559 & \textbf{0.4161}\\
Contact precision [\%] $\uparrow$ & 57.2589 & 56.9365 & 57.1576 & \underline{57.4316} & \textbf{59.1374}\\
interface\_\allowbreak CA\_\allowbreak distance\_\allowbreak preservation $\uparrow$ & 0.6838 & 0.6821 & \underline{0.6884} & 0.6849 & \textbf{0.7013}\\
internal\_\allowbreak CA\_\allowbreak RMSD $\downarrow$ & 1.8388 & 1.8547 & \underline{1.8193} & 1.8356 & \textbf{1.7242}\\
internal\_\allowbreak CA\_\allowbreak distance\_\allowbreak preservation $\uparrow$ & 0.7859 & 0.7841 & \underline{0.7887} & 0.7866 & \textbf{0.7976}\\
\end{longtable}\endgroup
\newpage
\par\ifdim\dimexpr\pagegoal-\pagetotal\relax<120pt\newpage\fi\subsection{Peptide design on LNR}
\begingroup\fontsize{8}{9.5}\selectfont\setlength{\tabcolsep}{2.5pt}\renewcommand{\arraystretch}{1.08}
\setlength{\LTpre}{12pt}\setlength{\LTpost}{6pt}
\begin{longtable}{@{}>{\raggedright\arraybackslash}p{.30\linewidth}*{4}{>{\centering\arraybackslash}p{.11\linewidth}}>{\centering\arraybackslash}p{.20\linewidth}@{}}
\caption{Peptide design on LNR: complete metrics.}\label{tab:full-lnr}\\\toprule
{\fontsize{9}{10.5}\selectfont Metric} & {\fontsize{9}{10.5}\selectfont AnewOmni} & {\fontsize{9}{10.5}\selectfont LLM4GEN} & {\fontsize{9}{10.5}\selectfont EnzyAdapter} & {\fontsize{9}{10.5}\selectfont IP-Adapter} & {\fontsize{9}{10.5}\selectfont \mbox{PairFuse (Ours)}}\\\midrule\endfirsthead
\multicolumn{6}{c}{Peptide design on LNR (continued)}\\\toprule
{\fontsize{9}{10.5}\selectfont Metric} & {\fontsize{9}{10.5}\selectfont AnewOmni} & {\fontsize{9}{10.5}\selectfont LLM4GEN} & {\fontsize{9}{10.5}\selectfont EnzyAdapter} & {\fontsize{9}{10.5}\selectfont IP-Adapter} & {\fontsize{9}{10.5}\selectfont \mbox{PairFuse (Ours)}}\\\midrule\endhead
\midrule\multicolumn{6}{r}{Continued on next page}\\\endfoot\bottomrule\endlastfoot
Complex RMSD [\AA{}] $\downarrow$ & 8.5619 & \underline{8.5006} & 8.5859 & 8.5664 & \textbf{8.2821}\\
Peptide RMSD [\AA{}] $\downarrow$ & 2.9091 & 2.8684 & \underline{2.8447} & 2.9068 & \textbf{2.8432}\\
AAR [\%] $\uparrow$ & 6.4306 & \underline{6.4805} & 6.4383 & \textbf{6.5084} & 6.3436\\
Intra-region clash [\%] $\downarrow$ & 1.0151 & 0.8918 & \textbf{0.8438} & 0.9654 & \underline{0.8452}\\
Interface clash [\%] $\downarrow$ & 0.8863 & \textbf{0.7135} & 0.8847 & 0.9950 & \underline{0.8026}\\
DockQ [0--1] $\uparrow$ & 0.4280 & \underline{0.4302} & 0.4272 & 0.4248 & \textbf{0.4383}\\
iRMSD [\AA{}] $\downarrow$ & 3.1039 & \underline{3.0790} & 3.1104 & 3.1030 & \textbf{3.0039}\\
Native-contact recall [\%] $\uparrow$ & 33.1855 & \underline{33.2962} & 32.9352 & 32.7373 & \textbf{34.0958}\\
Backbone interface lDDT [0--1] $\uparrow$ & 0.4197 & \underline{0.4213} & 0.4193 & 0.4169 & \textbf{0.4291}\\
structural/\allowbreak common\_\allowbreak region\_\allowbreak CA  & 10.1935 & 10.1935 & 10.1935 & 10.1935 & 10.1935\\
structural/\allowbreak common\_\allowbreak target\_\allowbreak CA  & 279.4624 & 279.4624 & 279.4624 & 279.4624 & 279.4624\\
structural/\allowbreak region\_\allowbreak CA\_\allowbreak count  & 10.4731 & 10.4731 & 10.4731 & 10.4731 & 10.4731\\
structural/\allowbreak region\_\allowbreak block\_\allowbreak count  & 10.4731 & 10.4731 & 10.4731 & 10.4731 & 10.4731\\
Contact F1 [0--1] $\uparrow$ & 0.3413 & \underline{0.3423} & 0.3386 & 0.3369 & \textbf{0.3487}\\
LRMSD [\AA{}] $\downarrow$ & 8.5605 & \underline{8.4984} & 8.5815 & 8.5603 & \textbf{8.2761}\\
clashes $\downarrow$ & 0.3900 & \underline{0.2995} & 0.3568 & 0.3645 & \textbf{0.2933}\\
complex\_\allowbreak CA\_\allowbreak RMSD $\downarrow$ & 8.5619 & \underline{8.5006} & 8.5859 & 8.5664 & \textbf{8.2821}\\
complex\_\allowbreak CA\_\allowbreak RMSD\_\allowbreak le\_\allowbreak 2A $\uparrow$ & 0.0688 & \underline{0.0722} & 0.0663 & 0.0649 & \textbf{0.0800}\\
Contact precision [\%] $\uparrow$ & 35.8817 & \underline{35.9828} & 35.6403 & 35.4452 & \textbf{36.4300}\\
interface\_\allowbreak CA\_\allowbreak distance\_\allowbreak preservation $\uparrow$ & 0.4149 & \underline{0.4161} & 0.4142 & 0.4119 & \textbf{0.4239}\\
internal\_\allowbreak CA\_\allowbreak RMSD $\downarrow$ & 2.9091 & 2.8684 & \underline{2.8447} & 2.9068 & \textbf{2.8432}\\
internal\_\allowbreak CA\_\allowbreak distance\_\allowbreak preservation $\uparrow$ & 0.7114 & 0.7148 & \underline{0.7153} & 0.7117 & \textbf{0.7175}\\
\end{longtable}\endgroup
\par\ifdim\dimexpr\pagegoal-\pagetotal\relax<120pt\newpage\fi\subsection{Small-molecule design on CrossDocked2020}
\begingroup\fontsize{8}{9.5}\selectfont\setlength{\tabcolsep}{2.5pt}\renewcommand{\arraystretch}{1.08}
\setlength{\LTpre}{12pt}\setlength{\LTpost}{6pt}
\begin{longtable}{@{}>{\raggedright\arraybackslash}p{.30\linewidth}*{4}{>{\centering\arraybackslash}p{.11\linewidth}}>{\centering\arraybackslash}p{.20\linewidth}@{}}
\caption{Small-molecule design on CrossDocked2020: complete metrics.}\label{tab:full-crossdocked2020}\\\toprule
{\fontsize{9}{10.5}\selectfont Metric} & {\fontsize{9}{10.5}\selectfont AnewOmni} & {\fontsize{9}{10.5}\selectfont LLM4GEN} & {\fontsize{9}{10.5}\selectfont EnzyAdapter} & {\fontsize{9}{10.5}\selectfont IP-Adapter} & {\fontsize{9}{10.5}\selectfont \mbox{PairFuse (Ours)}}\\\midrule\endfirsthead
\multicolumn{6}{c}{Small-molecule design on CrossDocked2020 (continued)}\\\toprule
{\fontsize{9}{10.5}\selectfont Metric} & {\fontsize{9}{10.5}\selectfont AnewOmni} & {\fontsize{9}{10.5}\selectfont LLM4GEN} & {\fontsize{9}{10.5}\selectfont EnzyAdapter} & {\fontsize{9}{10.5}\selectfont IP-Adapter} & {\fontsize{9}{10.5}\selectfont \mbox{PairFuse (Ours)}}\\\midrule\endhead
\midrule\multicolumn{6}{r}{Continued on next page}\\\endfoot\bottomrule\endlastfoot
Vina score [kcal/\allowbreak mol] $\downarrow$ & \underline{-5.1279} & -5.0685 & -5.1230 & -5.1040 & \textbf{-5.1517}\\
Vina minimize [kcal/\allowbreak mol] $\downarrow$ & \underline{-5.9715} & -5.9298 & -5.9478 & -5.9130 & \textbf{-5.9841}\\
Vina dock [kcal/\allowbreak mol] $\downarrow$ & -7.0580 & \textbf{-7.0891} & -7.0700 & -7.0646 & \underline{-7.0868}\\
QED [0--1] $\uparrow$ & 0.5745 & \textbf{0.5765} & \underline{0.5753} & 0.5675 & 0.5721\\
SA normalized [0--1] $\uparrow$ & 0.7684 & 0.7699 & \textbf{0.7771} & 0.7699 & \underline{0.7734}\\
Lipinski [/\allowbreak 5] $\uparrow$ & \underline{4.8066} & 4.7967 & 4.7935 & \textbf{4.8077} & 4.7759\\
chemical/\allowbreak logp  & 2.1805 & 2.2554 & 2.2341 & 2.1552 & 2.2446\\
geometry/\allowbreak atom\_\allowbreak stable\_\allowbreak num  & 19.1098 & 19.1661 & 18.8761 & 18.8311 & 19.3131\\
geometry/\allowbreak inter\_\allowbreak clash\_\allowbreak atom\_\allowbreak num $\downarrow$ & 0.0310 & \underline{0.0254} & \textbf{0.0246} & 0.0256 & 0.0270\\
geometry/\allowbreak inter\_\allowbreak clash\_\allowbreak detected $\downarrow$ & 0.0266 & 0.0249 & \textbf{0.0229} & \underline{0.0239} & 0.0249\\
geometry/\allowbreak intra\_\allowbreak clash\_\allowbreak atom\_\allowbreak num $\downarrow$ & 20.8222 & 20.8821 & \textbf{20.5754} & \underline{20.5862} & 21.0447\\
geometry/\allowbreak mol\_\allowbreak stable $\uparrow$ & \textbf{0.3295} & \underline{0.3274} & 0.3257 & 0.3095 & 0.3088\\
interaction/\allowbreak num\_\allowbreak halogen  & 0.0300 & 0.0378 & 0.0396 & 0.0346 & 0.0383\\
interaction/\allowbreak num\_\allowbreak hydrogen  & 2.2633 & 2.2328 & 2.2416 & 2.1802 & 2.2381\\
interaction/\allowbreak num\_\allowbreak hydrophobic  & 3.2538 & 3.2202 & 3.3403 & 3.3032 & 3.3776\\
interaction/\allowbreak num\_\allowbreak metal  & 0.0000 & 0.0000 & 0.0000 & 0.0000 & 0.0000\\
interaction/\allowbreak num\_\allowbreak pi\_\allowbreak cation  & 0.0862 & 0.0898 & 0.0890 & 0.0862 & 0.0936\\
interaction/\allowbreak num\_\allowbreak pi\_\allowbreak stack  & 0.3663 & 0.3633 & 0.3588 & 0.3365 & 0.3628\\
interaction/\allowbreak num\_\allowbreak wb  & 0.0000 & 0.0000 & 0.0000 & 0.0000 & 0.0000\\
substructure/\allowbreak atom\_\allowbreak type/\allowbreak 15  & 1.2286 & 1.1303 & 1.0989 & 1.1779 & 1.0898\\
substructure/\allowbreak atom\_\allowbreak type/\allowbreak 16  & 1.1647 & 1.1715 & 1.1287 & 1.1531 & 1.1422\\
substructure/\allowbreak atom\_\allowbreak type/\allowbreak 17  & 1.1200 & 1.1702 & 1.1749 & 1.1640 & 1.1831\\
substructure/\allowbreak atom\_\allowbreak type/\allowbreak 35  & 1.0000 & 1.0174 & 1.0100 & 1.0058 & 1.0000\\
substructure/\allowbreak atom\_\allowbreak type/\allowbreak 5  & 1.0000 & 1.0000 & 1.0278 & 1.0000 & 1.0500\\
substructure/\allowbreak atom\_\allowbreak type/\allowbreak 53  & 1.0000 & 1.0000 & 1.0000 & 1.0000 & 1.0000\\
substructure/\allowbreak atom\_\allowbreak type/\allowbreak 6  & 15.1064 & 15.1600 & 15.0284 & 14.8997 & 15.3570\\
substructure/\allowbreak atom\_\allowbreak type/\allowbreak 7  & 3.0425 & 3.0114 & 2.8763 & 2.9830 & 2.9325\\
substructure/\allowbreak atom\_\allowbreak type/\allowbreak 8  & 2.8065 & 2.7875 & 2.8059 & 2.7781 & 2.8410\\
substructure/\allowbreak atom\_\allowbreak type/\allowbreak 9  & 1.8372 & 1.8361 & 1.9381 & 2.0365 & 1.9328\\
substructure/\allowbreak functional\_\allowbreak group/\allowbreak N=CN  & 1.0000 & 1.0000 & 1.0000 & 1.0062 & 1.0090\\
substructure/\allowbreak functional\_\allowbreak group/\allowbreak NC(=O)O  & 1.0000 & 1.0000 & 1.0000 & 1.0000 & 1.0000\\
substructure/\allowbreak functional\_\allowbreak group/\allowbreak NC(N)=O  & 1.0000 & 1.0132 & 1.0071 & 1.0199 & 1.0104\\
substructure/\allowbreak functional\_\allowbreak group/\allowbreak NC=O  & 1.1318 & 1.1284 & 1.1607 & 1.1595 & 1.1423\\
substructure/\allowbreak functional\_\allowbreak group/\allowbreak NS(=O)=O  & 1.0000 & 1.0000 & 1.0083 & 1.0000 & 1.0000\\
substructure/\allowbreak functional\_\allowbreak group/\allowbreak O=CNO  & 1.0000 & 1.0000 & 1.0133 & 1.0139 & 1.0392\\
substructure/\allowbreak functional\_\allowbreak group/\allowbreak O=CO  & 1.0646 & 1.0690 & 1.0589 & 1.0641 & 1.0650\\
substructure/\allowbreak functional\_\allowbreak group/\allowbreak O=P(O)(O)O  & 1.1429 & 1.0833 & 1.0000 & 1.0000 & 1.0000\\
substructure/\allowbreak functional\_\allowbreak group/\allowbreak O=P(O)O  & 1.0000 & 1.0500 & 1.0000 & 1.1111 & 1.0000\\
substructure/\allowbreak functional\_\allowbreak group/\allowbreak O=S=O  & 1.0000 & 1.0000 & 1.0000 & 1.0000 & 1.0000\\
substructure/\allowbreak functional\_\allowbreak group/\allowbreak O=c1cc[nH]c(=O)[nH]1  & 1.0000 & 1.0000 & 1.0219 & 1.0139 & 1.0093\\
substructure/\allowbreak functional\_\allowbreak group/\allowbreak OCO  & 1.0541 & 1.0239 & 1.0000 & 1.0311 & 1.0655\\
substructure/\allowbreak functional\_\allowbreak group/\allowbreak c1c[nH]cn1  & 1.0106 & 1.0320 & 1.0226 & 1.0080 & 1.0177\\
substructure/\allowbreak functional\_\allowbreak group/\allowbreak c1ccc2[nH]ccc2c1  & 1.0444 & 1.0204 & 1.0293 & 1.0403 & 1.0186\\
substructure/\allowbreak functional\_\allowbreak group/\allowbreak c1ccc2[nH]cnc2c1  & 1.0000 & 1.0000 & 1.0093 & 1.0000 & 1.0090\\
substructure/\allowbreak functional\_\allowbreak group/\allowbreak c1ccc2ccccc2c1  & 1.0000 & 1.0000 & 1.0000 & 1.0000 & 1.0210\\
substructure/\allowbreak functional\_\allowbreak group/\allowbreak c1ccc2ncccc2c1  & 1.0185 & 1.0187 & 1.0197 & 1.0269 & 1.0294\\
substructure/\allowbreak functional\_\allowbreak group/\allowbreak c1ccccc1  & 1.3758 & 1.3767 & 1.4198 & 1.3726 & 1.4212\\
substructure/\allowbreak functional\_\allowbreak group/\allowbreak c1ccncc1  & 1.0498 & 1.0334 & 1.0341 & 1.0338 & 1.0470\\
substructure/\allowbreak functional\_\allowbreak group/\allowbreak c1ccsc1  & 1.0130 & 1.0125 & 1.0288 & 1.0221 & 1.0142\\
substructure/\allowbreak functional\_\allowbreak group/\allowbreak c1cn[nH]c1  & 1.0361 & 1.0181 & 1.0290 & 1.0170 & 1.0464\\
substructure/\allowbreak functional\_\allowbreak group/\allowbreak c1cncnc1  & 1.0190 & 1.0718 & 1.0486 & 1.0465 & 1.0495\\
substructure/\allowbreak functional\_\allowbreak group/\allowbreak c1cscn1  & 1.0147 & 1.0000 & 1.0000 & 1.0142 & 1.0235\\
substructure/\allowbreak functional\_\allowbreak group/\allowbreak c1ncc2nc[nH]c2n1  & 1.0000 & 1.0000 & 1.0000 & 1.0000 & 1.0000\\
substructure/\allowbreak ring\_\allowbreak type/\allowbreak 10  & 1.0000 & 1.0000 & 1.0000 & 1.0000 & 1.1037\\
substructure/\allowbreak ring\_\allowbreak type/\allowbreak 11  & 1.5000 & 1.0000 & 1.0000 & 1.0667 & 1.1111\\
substructure/\allowbreak ring\_\allowbreak type/\allowbreak 3  & 1.0000 & 1.0000 & 1.0000 & 1.0000 & 1.0000\\
substructure/\allowbreak ring\_\allowbreak type/\allowbreak 4  & 1.0000 & 1.0000 & 1.0000 & 1.0000 & 1.0000\\
substructure/\allowbreak ring\_\allowbreak type/\allowbreak 5  & 1.3451 & 1.3551 & 1.3121 & 1.3323 & 1.3246\\
substructure/\allowbreak ring\_\allowbreak type/\allowbreak 6  & 2.0944 & 2.0462 & 2.0433 & 2.0272 & 2.0803\\
substructure/\allowbreak ring\_\allowbreak type/\allowbreak 7  & 1.0435 & 1.0442 & 1.0410 & 1.0371 & 1.0501\\
substructure/\allowbreak ring\_\allowbreak type/\allowbreak 8  & 1.0000 & 1.0147 & 1.0215 & 1.0295 & 1.0314\\
substructure/\allowbreak ring\_\allowbreak type/\allowbreak 9  & 1.0000 & 1.0000 & 1.0303 & 1.0256 & 1.0476\\
2D Diversity [0--1] $\uparrow$ & 0.9155 & 0.9153 & 0.9149 & \textbf{0.9169} & \underline{0.9156}\\
High Affinity [\%] $\uparrow$ & 44.8174 & 43.7900 & \underline{45.8629} & 44.4766 & \textbf{46.8298}\\
Atom-type MAE $\downarrow$ & 0.4188 & \underline{0.4153} & 0.4404 & 0.4551 & \textbf{0.3775}\\
\end{longtable}\endgroup

\clearpage
\def\arxivfont{\rm}
\bibliographystyle{plainnat}

\bibliography{refs}

\end{document}